\documentclass{article}

\usepackage{microtype}
\usepackage{graphicx}
\usepackage{subcaption}
\usepackage{booktabs} 
\usepackage[framemethod=tikz]{mdframed}
\usepackage[nointegrals]{wasysym} 
\usepackage{xcolor}  
\usepackage{hyperref}
\usepackage{fontawesome5}

\usepackage{tikz}

\newcommand{\bluecircle}{\tikz[baseline=-0.5ex]\draw[blue,fill=blue,opacity=0.3] (0,0) circle (0.5ex);}

\usepackage[accepted]{icml2026}

\usepackage{enumitem}
\usepackage{amsmath}
\usepackage{amssymb}
\usepackage{mathtools}
\usepackage{amsthm}
\usepackage{thmtools}
\usepackage{thm-restate}



\usepackage{graphicx}
\let\oldtextsc\textsc
\renewcommand{\textsc}[1]{\scalebox{0.93}{\oldtextsc{#1}}}

\usepackage{algpseudocode}
\usepackage{algorithm}

\usepackage[capitalize,noabbrev]{cleveref}

\theoremstyle{plain}
\newtheorem{theorem}{Theorem}[section]
\newtheorem{prop}[theorem]{Proposition}

\theoremstyle{definition}

\theoremstyle{remark}

\newcommand{\R} {\mathbb R}

\newcommand{\dXt}{\mathrm{d}X_t}

\newcommand{\dt}{\mathrm{d}t}

\newcommand{\dx}{\mathrm{d}x}
\newcommand{\dWt}{\mathrm{d}W_t}

\DeclareMathSymbol{\shortminus}{\mathbin}{AMSa}{"39}

\usepackage[disable,textsize=tiny]{todonotes}

\usepackage[skip=3pt]{caption}
\usepackage{booktabs}
\usepackage{multirow}
\usepackage{graphicx} 
\usepackage{tikz}
\usetikzlibrary{positioning,calc,backgrounds,fit,arrows.meta, bending, decorations.markings, shapes.geometric}
\definecolor{popblue}{HTML}{4A90E2}

\definecolor{forecastcolour}{HTML}{8601AF}
\definecolor{costOrange}{HTML}{F5A623}
\definecolor{costRed}{HTML}{E02020}
\tikzset{
    boid/.pic={
        \draw[fill=popblue, thin, rotate=#1] (0,0) -- (-0.3, 0.15) -- (-0.2, 0) -- (-0.3, -0.15) -- cycle;
    },
    frameObserved/.style={
        draw=gray!80, thick
    },
    frameUnobserved/.style={
        draw=gray!50, thick, dashed
    },
    frameForecast/.style={
        draw=forecastcolour, thick, dashed
    },
    labelObserved/.style={
        font=\bfseries\Large, text=black
    },
    labelUnobserved/.style={
        font=\bfseries\Large, text=gray!70
    },
    labelForecast/.style={
        font=\bfseries\Large, text=forecastcolour
    }
}

\tikzset{
    boid/.pic={
        \draw[fill=popblue, thin, rotate=#1] (0,0) -- (-0.3, 0.15) -- (-0.2, 0) -- (-0.3, -0.15) -- cycle;
    },
    frameObserved/.style={
        draw=gray!80, thick
    },
    frameUnobserved/.style={
        draw=gray!50, thick, dashed 
    },
    labelObserved/.style={
        font=\bfseries\Large, text=black
    },
    labelUnobserved/.style={
        font=\bfseries\Large, text=gray!70 
    }
}

\newcommand{\flock}[3]{
    \begin{scope}[shift={(#1)}]
        \foreach \k in {1,...,#3} {
            \pgfmathsetmacro{\dx}{rand*0.4} 
            \pgfmathsetmacro{\dy}{rand*0.4}
            \pgfmathsetmacro{\dangle}{rand*30}
            \path (\dx, \dy) pic {boid={#2+\dangle}};
        }
    \end{scope}
}

\icmltitlerunning{A Call to Lagrangian Action: Learning Population Mechanics from Temporal Snapshots}

\begin{document}

\twocolumn[
  \icmltitle{A Call to Lagrangian Action: \texorpdfstring{\\}{ } Learning Population Mechanics from Temporal Snapshots}



  \icmlsetsymbol{equal}{*}

  \begin{icmlauthorlist}
  \icmlauthor{Vincent Guan}{ubc,mila}
  \icmlauthor{Lazar Atanackovic}{uofa,amii,broad}
  \icmlauthor{Kirill Neklyudov}{mila,udem,courtois}
\end{icmlauthorlist}

\icmlaffiliation{ubc}{University of British Columbia}

\icmlaffiliation{broad}{Broad Institute of MIT and Harvard}
\icmlaffiliation{uofa}{University of Alberta}
\icmlaffiliation{amii}{Alberta Machine Intelligence Institute}

\icmlaffiliation{mila}{Mila - Quebec AI Institute}
\icmlaffiliation{udem}{Universit\'e de Montr\'eal}
\icmlaffiliation{courtois}{Institut Courtois}
\vspace{2em}


\icmlcorrespondingauthor{Vincent Guan}{vincentguan23@gmail.com}
\icmlcorrespondingauthor{Kirill Neklyudov}{k.necludov@gmail.com}


]



\printAffiliationsAndNotice{}  

\begin{abstract}
The population dynamics of molecules, cells, and organisms are governed by a number of unknown forces. In the last decade, population dynamics have predominantly been modeled with Wasserstein gradient flows.
However, since gradient flows minimize free energy, they fail to capture important dynamical properties, such as periodicity. In this work, we propose a change in perspective by considering dynamics that minimize a \emph{population-level action} under a damped Wasserstein Lagrangian. By deriving the corresponding Hamiltonian equations of motion, we formalize \emph{Wasserstein Lagrangian Mechanics}, a structured class of second-order dynamics that encompasses classical mechanics, quantum mechanics, and gradient flows. We then propose \texttt{WLM} as the first algorithm that learns these second-order dynamics from observed marginals, without specifying the Lagrangian. By directly learning the population mechanics, \texttt{WLM} can both forecast and interpolate unseen marginals, and outperforms existing gradient flow and flow matching methods across a wide range of dynamics, including vortex dynamics, embryonic development, and flocking.
\end{abstract}

\vspace{-2em}

\section{Introduction}
Population dynamics govern every level of nature, from the Brownian motion of molecules in a fluid, to the embryonic differentiation of developing cells, to the swarm behaviour of animals \citep{may1987chaos,codling2008random}. Identifying the underlying forces that drive the evolution of a given population is therefore a fundamental, but difficult problem \citep{newman2014modelling,agozzino2020cells}.

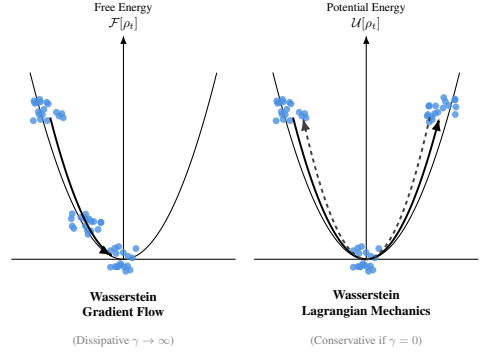
\begin{figure}[t]
    \centering
    \resizebox{0.75\columnwidth}{!}{%
    \begin{tikzpicture}[
        >=Latex, 
        thick,
        pics/population/.style={
            code={
                \begin{scope}[local bounding box=pop]
                    \foreach \i in {1,...,18} {
                        \fill[popblue, opacity=0.8] (rand*0.42, rand*0.34) circle (2.6pt);
                    }
                \end{scope}
            }
        }
    ]

        \begin{scope}
            \draw[-] (-3,0) -- (3,0);
            \draw[->] (0,0) -- (0,6) 
                node[above, align=center] {\small Free Energy \\ $\mathcal{F}[\rho_t]$}; 
            
            \pgfmathsetmacro{\wL}{2.5}
            \pgfmathsetmacro{\kL}{5/(\wL*\wL)}
            \draw[thick] plot[domain=-\wL:\wL, samples=200] (\x, {\kL*\x*\x});
            
            \pgfmathsetmacro{\wLArrow}{2.25}
            \pgfmathsetmacro{\kLArrow}{5/(\wLArrow*\wLArrow)}
            
            \pgfmathsetseed{100}
            \path (-2, 4) pic {population};
            
            \pgfmathsetseed{200}
            \path (-1, 1) pic {population};
            
            \pgfmathsetseed{300}
            \path (0, 0) pic {population};

            \draw[->, line width=1.5pt] 
    plot[domain=-1.96:-0.30, samples=150] (\x, {\kLArrow*\x*\x});
            
            \node[font=\bfseries, align=center] at (0,-1.2) {Wasserstein \\ Gradient Flow};
            \node[text=gray, font=\small, align=center] at (0,-2.2) {(Dissipative $\gamma \to \infty$)};
        \end{scope}

        \begin{scope}[xshift=6.5cm] 
            \draw[-] (-3,0) -- (3,0);
            \draw[->] (0,0) -- (0,6) 
                node[above, align=center] {\small Potential Energy \\ $\mathcal{U}[\rho_t]$}; 
            
            \pgfmathsetmacro{\wOut}{2.5}
            \pgfmathsetmacro{\kOut}{5/(\wOut*\wOut)}
            \draw[thick] plot[domain=-\wOut:\wOut, samples=220] (\x, {\kOut*\x*\x});
            
            \pgfmathsetmacro{\wSolid}{2.25}   
            \pgfmathsetmacro{\wDashed}{1.95}  
            \pgfmathsetmacro{\kSolid}{5/(\wSolid*\wSolid)}
            \pgfmathsetmacro{\kDashed}{5/(\wDashed*\wDashed)}
            
            \pgfmathsetseed{100}
            \path (-2, 4) pic {population};
            
            \pgfmathsetseed{300}
            \path (0, 0) pic {population};
            
            \pgfmathsetseed{400}
            \path (2, 4) pic {population};


            \draw[->, line width=1.5pt] 
                plot[domain=-1.96:1.96, samples=200] (\x, {\kSolid*\x*\x});

            \draw[->, dashed, darkgray, line width=1.5pt] 
                plot[domain=1.70:-1.70, samples=200] (\x, {\kDashed*\x*\x});

            \node[font=\bfseries, align=center] at (0,-1.2) {Wasserstein \\ Lagrangian Mechanics};
            \node[text=gray, font=\small, align=center] at (0,-2.2) {(Conservative if $\gamma = 0$)};
        \end{scope}

    \end{tikzpicture}%
    }
    \caption{Wasserstein gradient flows describe first-order population dynamics that minimize the free energy $\mathcal{F}[\rho_t]$. We propose Wasserstein Lagrangian mechanics (WLM), which describe a richer class of damped second-order dynamics, based on the population-level potential energy $\mathcal{U}[\rho_t]$. Given the same quadratic functional, gradient flows dissipate until equilibrium, while WLM produces oscillating dynamics if $\gamma =0$, and degenerate to gradient flow dynamics in the overdamped limit $\gamma \to \infty$.}
    \label{fig:dynamics_comparison}
\end{figure}
If domain knowledge is available, then a prior can be prescribed on the dynamics \citep{shen2024multi,simpson2025inference}, but in most cases, we are unable to obtain an accurate reference process. The predominant approach in machine learning and computational biology has been to model population dynamics as a gradient flow \citep{hashimoto2016learning, weinreb2018fundamental, lavenant2021towards, bunne2022proximal, terpin2024learning, guangradient}. Indeed, gradient flows have a rich mathematical foundation \citep{ambrosio2008gradient} and admit optimization schemes for tractable inference \citep{jordan1998variational}. However, since gradient flows minimize free energy, they can only characterize purely dissipative aperiodic dynamics. Flow-based methods \citep{atanackovic2024meta, kapusniak2024metric, petrovic2025curly} have been developed to address these limitations and produce higher quality interpolations, but they cannot forecast beyond the observed horizon.

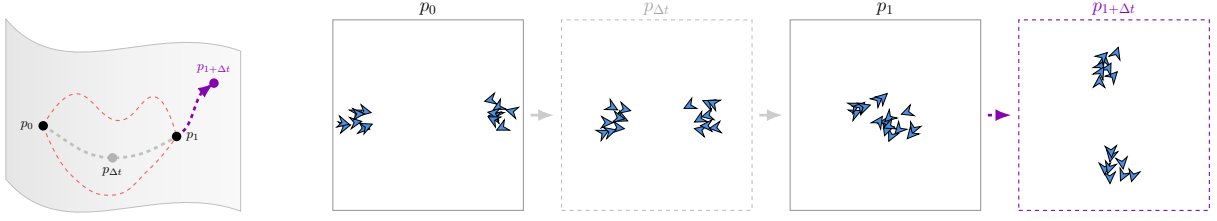
\begin{figure*}[t]
    \centering
    \resizebox{\linewidth}{!}{%
    \begin{tikzpicture}[>=Latex]
        
        
        \tikzset{
            boid/.pic={
                \draw[fill=popblue, thin, rotate=#1] (0,0) -- (-0.3, 0.15) -- (-0.2, 0) -- (-0.3, -0.15) -- cycle;
            },
            frameObserved/.style={draw=gray!80, thick, rectangle},
            frameUnobserved/.style={draw=gray!50, thick, dashed, rectangle},
            frameForecast/.style={draw=forecastcolour, thick, dashed, rectangle},
            labelObserved/.style={font=\bfseries\Large, text=black},
            labelUnobserved/.style={font=\bfseries\Large, text=gray!70},
            labelForecast/.style={font=\bfseries\Large, text=forecastcolour},
            manifoldSurf/.style={shade, left color=gray!20, right color=gray!5, middle color=gray!10, draw=gray!50, thin},
            manifoldPoint/.style={circle, fill=black, inner sep=2.5pt},
            perturbationPath/.style={-, red!80, thick, dashed, opacity=0.8}
        }

        \def\framesize{5}
        \def\xspacing{6.0} 

        
        \begin{scope}[scale=1.4, local bounding box=manifoldBox]
            
            \coordinate (TL) at (-2.2, 1.8);
            \coordinate (TR) at (2.2, 1.2);
            \coordinate (BR) at (2.2, -1.8);
            \coordinate (BL) at (-2.2, -1.2);

            \draw[manifoldSurf] 
                (TL) to[out=-50, in=160] (TR)   
                     -- (BR)                    
                     to[out=160, in=-50] (BL)   
                     -- cycle;                  

            \coordinate (m0) at (-1.5, -0.2); 
            \coordinate (m1) at (-0.2, -0.8); 
            \coordinate (m2) at (1.0, -0.4);  
            \coordinate (m3) at (1.7, 0.6);   

            \draw[perturbationPath] plot [smooth, tension=0.8] coordinates {
                (m0) (-0.9, 0.4) (-0.1, -0.1) (0.6, 0.34) (m2)
            };
            
            \draw[perturbationPath] plot [smooth, tension=0.8] coordinates {
                (m0) (-1.0, -1.1) (-0.2, -1.5) (0.5, -1.0) (m2)
            };

            \draw[-, line width=2pt, dashed,gray!50] (m0) to[out=-30, in=180] (m1) to[out=0, in=210] (m2);
            
            \draw[->, line width=2pt, forecastcolour, dashed] (m2) to[out=30, in=200] (m3);

            \node[manifoldPoint, label={left:$p_0$}] at (m0) {};
            \node[manifoldPoint, fill=gray!60, label={below:$p_{\Delta t}$}] at (m1) {};
            \node[manifoldPoint, label={right:$p_1$}] at (m2) {};
            \node[manifoldPoint, fill=forecastcolour, label={above:\textcolor{forecastcolour}{$p_{1+\Delta t}$}}] at (m3) {};
        \end{scope}

        
        \def\panelStart{8.0} 

        \begin{scope}[xshift=\panelStart cm, local bounding box=panelsBox]

            \begin{scope}[xshift=0cm]
                \node[frameObserved, minimum size=\framesize cm] (p0) at (0,0) {};
                \node[above, labelObserved] at (p0.north) {$p_0$};
                \flock{-1.8, 0}{0}{8} \flock{1.8, 0}{180}{8}
            \end{scope}

            \begin{scope}[xshift=\xspacing cm]
                \node[frameUnobserved, minimum size=\framesize cm] (p1) at (0,0) {};
                \node[above, labelUnobserved] at (p1.north) {$p_{\Delta t}$};
                \flock{-1.0, 0}{0}{8} \flock{1.0, 0}{180}{8}
            \end{scope}

            \begin{scope}[xshift=2*\xspacing cm]
                \node[frameObserved, minimum size=\framesize cm] (p2) at (0,0) {};
                \node[above, labelObserved] at (p2.north) {$p_1$};
                \flock{-0.3, 0.3}{45}{8} \flock{0.3, -0.3}{225}{8}
            \end{scope}

            \begin{scope}[xshift=3*\xspacing cm]
                \node[frameForecast, minimum size=\framesize cm] (p3) at (0,0) {};
                \node[above, labelForecast] at (p3.north) {$p_{1+\Delta t}$};
                \flock{-0.1, 1.4}{70}{8} \flock{0.1, -1.4}{250}{8}
            \end{scope}

            \draw[->, line width=2pt, gray!40, shorten >=5pt, shorten <=5pt] (p0.east) -- (p1.west);
            \draw[->, line width=2pt, gray!40, shorten >=5pt, shorten <=5pt] (p1.east) -- (p2.west);
            \draw[->, line width=2pt, forecastcolour, dashed, shorten >=5pt, shorten <=5pt] (p2.east) -- (p3.west);

        \end{scope}

        
        \node[font=\bfseries\LARGE] (labelb) at ([yshift=-1.25cm]panelsBox.south) {(b) Wasserstein Hamiltonian mechanics 
        (visualized in $\mathbb{R}^d$)};

        \node[font=\bfseries\LARGE] at (manifoldBox.south |- labelb) {(a) Wasserstein Lagrangian flow in $\mathcal{P}(\mathbb{R}^d)$};
    \end{tikzpicture}%
    }
    \caption{\textbf{Learning population mechanics with WLM:} 
    In \textbf{(a)}, we illustrate the principle of least Wasserstein action \eqref{eq: minimize_lagrangian_action}: given observed marginals $p_0$ and $p_1$, the true interpolants form a minimal action curve in the space of densities, with respect to a population-level Lagrangian action. {Alternative curves of densities} have higher action and are drawn in red. Wasserstein least action induces Hamiltonian mechanics on the population, which we visualize in the state space for a population of interacting Boids in \textbf{(b)}. Our method (\texttt{WLM}) learns these mechanics to \textcolor{gray}{interpolate} and \textcolor{violet}{forecast} unseen marginals (dashed grey and violet panels).}
    \label{fig:boids_least_action_dynamics}
\end{figure*}

A change of perspective is therefore necessary to model more general population dynamics and to forecast future dynamics. As a natural starting point, we consider \emph{the principle of least Wasserstein action}: 
\begin{mdframed}[hidealllines=true,backgroundcolor=blue!8]
The population marginals $(p_t)_{0 \le t \le T}$ minimize the action of some population-level Lagrangian $\mathcal{L}[\rho_t, \dot{\rho}_t,t]$.
\end{mdframed}
Least action in the Wasserstein space of probability distributions has been proposed in recent literature \citep{chow2020wasserstein, neklyudov2023action}, and states that the population evolves according to the most efficient path, as determined by the population-level Lagrangian.
By defining the Lagrangian on the Wasserstein space, rather than the state space, this formulation captures emergent dynamics. For example, the evolution could be determined by interactions between individuals, as well as the population's intrinsic properties, like its central mass. Moreover, by working directly on the space of marginals, this framework is directly amenable to settings where it is impossible to track trajectories, such as single-cell biology. While least action is intuitively attractive, gradient flow and flow matching remain the prevailing machine learning approaches for learning population dynamics. To motivate least action as an advantageous framework, two key questions must be answered:
\begin{enumerate}[leftmargin=*]
    \item Does Wasserstein least action capture more expressive dynamics than Wasserstein gradient flows?
    \item How can least action dynamics be inferred from observational data, without specifying the true Lagrangian? 
\end{enumerate}
\paragraph{Contributions.}
In this work, we propose \emph{Wasserstein Lagrangian Mechanics} (WLM), which \emph{re-frames learning population dynamics as learning the population mechanics of an evolving population}. We prove that WLM is significantly more expressive than Wasserstein gradient flows (see \cref{fig:dynamics_comparison}), and complement our theory by developing an algorithm that learns the mechanics directly from data. 
We outline our contributions and the structure of our paper:
\begin{itemize}[leftmargin=*,noitemsep]
    \item \textbf{Theory:} In \cref{sec:wlm_theory}, we formalize WLM by deriving the Hamiltonian equations of motion from the least action principle under a damped Wasserstein Lagrangian. These equations of motion describe a structured class of population mechanics, which encompasses classical mechanics, quantum mechanics, and gradient flows.
    \item \textbf{Method:} 
    In \cref{sec:learning_wlm}, we propose \texttt{WLM}, a neural mechanics model that learns population mechanics directly from data. \texttt{WLM} learns the system's Hamiltonian mechanics by parameterizing the population-level potential energy and training the model to predict its next $k$ marginals.
    \item \textbf{Experiments: } In \cref{sec:experiments}, we demonstrate \texttt{WLM}'s ability to learn gradient flows, curling dynamics from ocean vortex data, cell dynamics from scRNA data, and emergent dynamics from Boids. By learning the population mechanics, \texttt{WLM} consistently outperforms state-of-the-art methods for
    both \textit{forecasting} and \textit{interpolation}, and does so without requiring a reference process.
\end{itemize}

\paragraph{Conflict of Interest Disclosure.} The authors declare no financial conflicts of interest.



\vspace{-1em}
\section{Wasserstein Lagrangian Mechanics}
\label{sec:wlm_theory}
In this section, we develop a mathematical framework for the mechanics of an evolving population. We first review population dynamics under the canonical continuity equation, and define population-level coordinates for flows that minimize a Wasserstein Lagrangian action. Then, from the Wasserstein least action principle, we determine the \emph{population mechanics} by deriving the Hamiltonian equations of motion. This characterizes the population as a damped second-order system, driven by the population-level potential energy $\mathcal{U}[\rho_t]$.
Finally, we use these insights to clarify the scope of dynamics that WLM encompasses, including gradient flows, as well as classical and quantum mechanics.

\vspace{-0.7em}
\subsection{Continuous Population Dynamics}
We recall an important theorem from \citet{ambrosio2008gradient}: continuous population dynamics in the Wasserstein-$2$ space $\mathcal{P}_2(\R^d)$ can always be represented by the continuity equation, with time-varying vector field $v_t = \nabla s_t$.
\begin{prop}[\text{\citet[Theorem 8.3.1]{ambrosio2008gradient}}]
\label{prop:unique_conservative_vector_field}Let $(\rho_t)_{0 \le t \le 1}$ be an absolutely continuous curve of measures in $\mathcal{P}_2(\R^d)$. Then, there exists a unique (up to a constant) potential $s_t: \R^d \to \R$ such that the population dynamics $(\rho_t)_{0 \le t \le 1}$ are defined by the evolution
\begin{align}
    \dot{\rho_t} = -\nabla \cdot (\rho_t \nabla s_t).
    \label{eq: continuity_eq}
\end{align}
Moreover, of all vector fields $v_t$, which generate $(\rho_t)_{0 \le t \le 1}$ via $\dot{\rho_t} = -\nabla \cdot (\rho_t v_t)$, $\nabla s_t$ is the one that minimizes the $W_2$ kinetic energy for a.e. $t \in [0,1]$:
\begin{align}
    \mathcal{K}[\rho_t, \dot{\rho}_t] = \frac12\int \|v_t\|^2 \rho_t dx.\label{eq:canonical_kinetic_energy}
\end{align}
\end{prop}
\cref{prop:unique_conservative_vector_field} offers a canonical representation for continuous population dynamics, but we note that population marginals can be produced by multiple laws on paths, unless identifiability conditions hold \citep{guan2024identifying, guangradient}. In this work, we assume that the population dynamics are absolutely continuous in $\mathcal{P}_2(\R^d)$ and use the canonical representation \eqref{eq: continuity_eq}. This allows us to conveniently rewrite the coordinates $(\rho_t, \dot{\rho}_t)$ as $(\rho_t, s_t)$, such that $\rho_t \in \mathcal{P}_2(\R^d)$ defines (Wasserstein) position and $\nabla s_t$ is a tangent vector \citep{ambrosio2008gradient}[Section 8.4].

\subsection{Wasserstein Least Action Population Dynamics}
\label{subsec:wlm_prelim}
We now consider population dynamics that minimize a Wasserstein Lagrangian action \citep{neklyudov2023computational}. As in classical mechanics, the simplest Wasserstein Lagrangian is the sum of the kinetic and (negative) potential energy functionals. The scalar energy values are defined at the population level, via the Wasserstein coordinates $(\rho_t, s_t)$: 
\begin{align*}
    \mathcal{L}[\rho_t, s_t] = \mathcal{K}[\rho_t, s_t] - \mathcal{U}[\rho_t].
\end{align*}
In this work, we fix $\mathcal{K}[\rho_t, s_t]$ as the $W_2$ kinetic energy \eqref{eq:canonical_kinetic_energy} with the canonical vector field $v_t = \nabla s_t$. The potential energy $\mathcal{U}[\rho_t]$ is the more crucial and flexible term, as it can be any functional of the density, including entropy, interaction kernels, and expectations. We also note that more complicated Lagrangians can explicitly depend on time. 

If we fix a Lagrangian $\mathcal{L}[\rho_t, s_t,t]$, and a set of marginals $\{p_{t_i}\}_{i=1}^{M}$, then the most efficient curve of marginals with respect to the Lagrangian action is the one that obeys
\begin{align}
    (p_t)_{0 \le t \le 1}  &= \arg \min_{\substack{\{\rho_t: [0,1] \to \mathcal{P}_2(\R^d) \\
    \label{eq: minimize_lagrangian_action}
    \text{s.t. } \rho_{t_i} = p_{t_i} \forall i=1,\ldots,M\}
    }}\mathcal{S}[\rho_t, s_t, t] \\
    \mathcal{S}[\rho_t, s_t, t]&=\int_0^1 \mathcal{L}[\rho_t, s_t,t] dt.
    \notag
\end{align}
We can also characterize the least action dynamics $(p_t)_{0 \le t \le 1}$ as having coordinates $(\rho_t, s_t)$ that obey the first-order conditions
\begin{align}
    \frac{\delta \mathcal{S}[\rho_t, s_t]}{\delta \rho_t}[h] = \frac{\delta \mathcal{S}[\rho_t, s_t]}{\delta s_t}[h] =0,
    \label{eq:first_order_optimality}
\end{align}
which holds for all test functions $h \in C_c^\infty(\R^d \times (0,1) )$. Note that this matches classical stationary action, but replaces pointwise derivatives with functional derivatives, since position and velocity are defined on $\mathcal{P}_2(\R^d)$.



\subsection{Hamiltonian Equations of Motion}
\label{sec: equations_of_motion}
We have seen that without loss of generality, population dynamics evolve according to the continuity equation \eqref{eq: continuity_eq} for some time-varying vector field $\nabla s_t$. A natural question to ask is if we can determine insights about the driving potential $s_t$. We can in fact derive its evolution equation under the principle of least Wasserstein action. 

\begin{theorem}[Hamiltonian equations of motion]
Let $(p_t)_{0 \le t \le 1}$ minimize action with respect to the damped Wasserstein Lagrangian 
\begin{align}
    \mathcal{L}[\rho_t, s_t,t] = e^{\gamma t} \left(\frac{1}{2}\int \|\nabla s_t\|^2 \rho_t dx - \mathcal{U}[\rho_t]\right)
    \label{eq:damped_lagrangian}
\end{align}
with damping $\gamma \ge 0$

Then, $(p_t)_{0 \le t \le 1}$ obeys the continuity equation \eqref{eq: continuity_eq}, such that the driving potential $s_t$ obeys 
\begin{align}
    \dot{s}_t(x)&= -\frac{1}{2}\|\nabla s_t(x)\|^2 - \frac{\delta \mathcal{U}[\rho_t]}{\delta \rho_t}(x) - \gamma s_t(x).\label{eq:potential_evolution}
\end{align}
The marginals $(p_t)_{0 \le t \le 1}$ can therefore be produced by sampling trajectories, obeying the second-order mechanics:
\begin{mdframed}[
    hidealllines=true, 
    backgroundcolor=gray!8, 
    nobreak=true,
    innertopmargin=0pt,    
    innerbottommargin=10pt  
]
\begin{align}
    \frac{d}{dt} x_t &= v_t \coloneq \nabla s_t(x_t), & x_0 \sim p_0 \\
    \frac{d}{dt} v_t &= -\nabla \frac{\delta \mathcal{U}[\rho_t]}{\delta \rho_t}(x_t) - \gamma v_t, & v_0 = \nabla s_0(x_0)
    \label{eq:equations_of_motion}
\end{align}
\end{mdframed}
\label{thm:hamiltonian_eqs_wlm}
\end{theorem}
\begin{proof}
We present the full derivation in \cref{sec:app_proofs}.
\end{proof}

The Hamiltonian equations of motion \eqref{eq:equations_of_motion} offer a complementary mechanical perspective to the principle of least Wasserstein action \eqref{eq: minimize_lagrangian_action}, which can be visualized in the state space (see \cref{fig:boids_least_action_dynamics}). Intuitively, the population obeys second-order mechanics under a {`generalized Newton's law'}: individual particles are accelerated by a force determined by the population-level potential energy $\mathcal{U}[\rho_t]$. This yields an elegant interpretation: under the principle of least action, the population drives itself, by using the potential energy of its current configuration to evolve its transport vector field $\nabla s_t$. While these mechanics produce the true marginals at each time, we note that multiple laws on paths explain the same population dynamics (see \cref{app:non_iden_traj}).

We note that similar Wasserstein Hamiltonian equations have been derived in the conservative regime \citep[Proposition 2]{chow2020wasserstein}. By additionally considering the damping parameter $\gamma \ge 0$, we characterize a more expressive class of dissipative population mechanics, which includes gradient flows, as we will see in the next section.

\subsection{Expressivity of Dynamics under WLM}
\label{sec:expressivity_wlm}
In this section, we show that the class of damped second-order dynamics described by Wasserstein Lagrangian Mechanics encompasses a range of fundamental dynamics. First, we review the undamped setting, which characterizes classical and quantum mechanics. 

\begin{prop}[WLM describes classical mechanics] If $\gamma = 0$ and $\mathcal{U}[\rho_t] = \int U(x)\rho_t(x)dx$, then WLM describes Newtonian mechanics.
\end{prop}
\vspace{-1em}
\begin{proof}
Under Newtonian mechanics, a particle's acceleration follows the gradient of a state-space potential energy, $\nabla U(x)$\citep{arnold1989mathematical}. We thus recover Newtonian mechanics from the Hamiltonian equations \eqref{eq:equations_of_motion} if $\gamma =0$ and the Wasserstein potential energy is an expectation. Indeed, $\mathcal{U}[\rho_t] = \int U(x)\rho_t(x)dx \implies \frac{\delta \mathcal{U}[\rho_t]}{\delta \rho_t}(x)=U(x)$.
\end{proof}
\vspace{-1em}
 
In other words, if $\mathcal{U}[\rho_t]$ acts linearly on the density via the state-space potential $U(x)$, then individual masses would evolve independently of the rest of the population. These functionals describe classical mechanics, including gravitational motion, harmonic oscillation, and rigid body rotation. 

By adding a nonlinear Fisher information term to the potential energy functional, WLM also describes the Schr\"odinger equation, which governs the wave function in quantum mechanics (see \citet[Example 3]{chow2020wasserstein} and \citet[Section B.3]{neklyudov2023computational}).

We now show that Wasserstein gradient flows admit two characterizations under WLM: a stable characterization in the limit of overdamped friction, and an unstable characterization as a conservative system, which requires a specific initial velocity.

\begin{prop}[WLM describes gradient flows] Let $(\rho_t)_{t \ge 0}$ minimize a free energy functional $\mathcal{F}[\rho]$. Then the dynamics $(\rho_t)_{t \ge 0}$ can be described by WLM as
\begin{enumerate}[noitemsep,nosep,leftmargin=*]
    \item The overdamped limit ($\gamma \to \infty$) of Wasserstein Lagrangian systems with $\mathcal{U}[\rho_t]=\mathcal{F}[\rho_t]$.
    \item A conservative system ($\gamma = 0$) with `inverted' potential energy $\mathcal{U}[\rho_t] = -\frac 12 \int  \|\nabla \frac{\delta}{\delta \rho_t} \mathcal{F}[\rho_t] \|^2 \rho_t dx$ and initial velocity field $v_0 \propto -\nabla \frac{\delta}{\delta \rho_0} \mathcal{F}[\rho_0]$.
\end{enumerate}
\label{prop:GF_as_WLM}
\end{prop}
\vspace{-1em}
\begin{proof}

To recover the interpretation of a gradient flow as the overdamped limit of Lagrangian systems (see \citet[Chapter 7]{villani2009optimal} and \citet{adams2013large}), but generalized to the Wasserstein space, we first recall that given the canonical coordinates, a gradient flow of $\mathcal{F}$ has vector field proportional to the Wasserstein gradient of the free energy. Formally, $\nabla s_t \propto -\nabla \frac{\delta}{\delta \rho_t}\mathcal{F}[\rho_t]$ holds $\rho_t$-a.e. 
\citep[(4.10)]{santambrogio2017euclidean}. Then, we consider the Hamiltonian equations \eqref{eq:equations_of_motion} with $\mathcal{U}[\rho_t]=\mathcal{F}[\rho_t]$ and take $\gamma \to \infty$. We may apply Tikhonov's theorem \citep{tikhonov1952systems} to obtain
\begin{align*}
    \frac{dx}{dt} \propto -\nabla \frac{\delta\mathcal{F}[\rho_t]}{\delta \rho_t}(x) \implies \nabla s_t \propto -\nabla \frac{\delta \mathcal{F}[\rho_t]}{\delta \rho_t}.
\end{align*}
The implication is from the principle of superposition \citep[Theorem 8.2.1]{ambrosio2008gradient}. It follows that the overdamped characterization produces gradient flow dynamics for $t > 0$, no matter what $v_0$ is initialized as.

Then, to prove the second description, we note that since a gradient flow has vector field $\nabla s_t \propto -\nabla \frac{\delta}{\delta \rho_t}\mathcal{F}[\rho_t]$, it follows that, between any marginals $p_0$ and $p_1$, the population dynamics minimize
\begin{align*}
    \mathcal{J}[\rho_t, s_t] = \frac12\int_0^1 \int \| \nabla s_t + \nabla \frac{\delta}{\delta \rho_t} \mathcal{F}[\rho_t]\|^2 \rho_t \ dx \  dt,
\end{align*}
We then simplify and obtain the equivalent minimization
\begin{align*}
    \mathcal{S}[\rho_t, s_t]=\int_0^1 \frac12 \int \| \nabla s_t\|^2 \rho_t  + \| \nabla \frac{\delta}{\delta \rho_t} \mathcal{F}[\rho_t]\|^2 \rho_t \ dx dt,
\end{align*}
which is the Lagrangian action with potential energy $\mathcal{U}[\rho_t] = -\frac 12 \int  \|\nabla \frac{\delta}{\delta \rho_t} \mathcal{F}[\rho_t] \|^2 \rho_t dx$. However, these dynamics only determine a gradient flow if we also have $\nabla s_0  \propto -\nabla \frac{\delta}{\delta \rho_0} \mathcal{F}[\rho_0]$. Otherwise, the gradient flow condition $\nabla s_t \propto -\nabla \frac{\delta}{\delta \rho_t}\mathcal{F}[\rho_t]$ fails to hold for $t>0$ (see \cref{prop:WGF_as_WLM}).
\end{proof}

\begin{figure}[t]
    \centering
    \begin{subfigure}[t]{0.48\linewidth}
        \includegraphics[width=\linewidth]{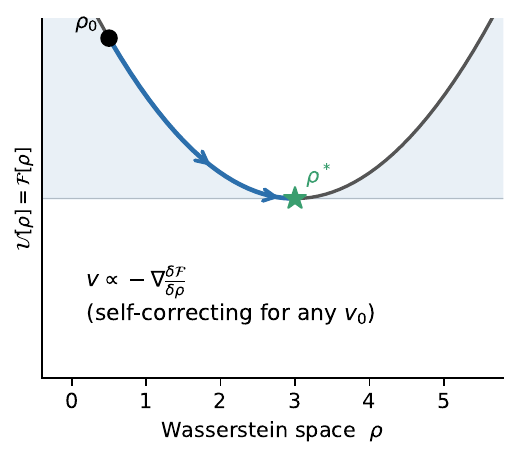}
        \caption{Stable overdamped system ($\gamma \to \infty$) with $\mathcal{U}[\rho] = \mathcal{F}[\rho]$}
    \end{subfigure}
    \hfill
    \begin{subfigure}[t]{0.48\linewidth}
        \includegraphics[width=\linewidth]{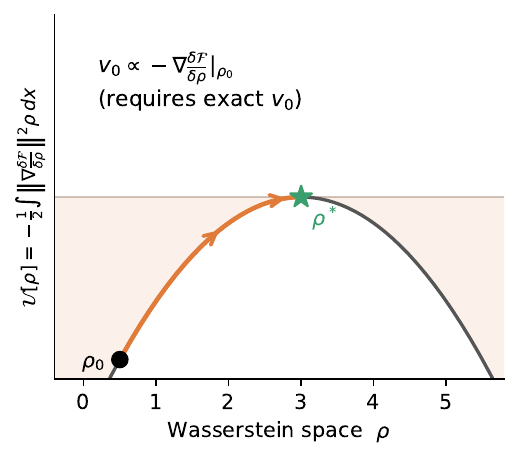}
        \caption{Unstable conservative system with inverted potential}
    \end{subfigure}
    \caption{We visualize Proposition~\ref{prop:GF_as_WLM}, which shows that Wasserstein gradient flows admit two characterizations under WLM.}
    \label{fig:prop_gf_as_wlm}
\end{figure}

We offer a visual interpretation of \cref{prop:GF_as_WLM} in \cref{fig:prop_gf_as_wlm}. Intuitively, while gradient flows are an expressive class of dynamics, popularly used for modeling population dynamics, \cref{prop:GF_as_WLM} shows that they only represent an extremely specific subset of Wasserstein Lagrangian Mechanics. Indeed, the free energy $\mathcal{F}[\rho]$ and the potential energy $\mathcal{U}[\rho]$ can each be arbitrary functionals, but gradient flows produce first-order mechanics along the free energy landscape determined by $\mathcal{F}[\rho]$, while WLM produces second-order mechanics determined by $\mathcal{U}[\rho]$ and $\gamma$. Given equality $\mathcal{U}[\rho]=\mathcal{F}[\rho]$, the WLM system only degenerates into a gradient flow in the overdamped limit. Similarly, \cref{prop:GF_as_WLM} shows that gradient flows can also be represented as highly unstable conservative systems, but only if the initial velocity obeys $v_0 \propto \nabla s_0$. Thus, any deviation in $v_0$ still characterizes a WLM system, but will fail to be a gradient flow of $\mathcal{F}$.

\section{Learning Population Mechanics}
\label{sec:learning_wlm}
Having demonstrated WLM's expressivity, we now turn to the task of inferring these mechanics from data. We note that if the Lagrangian $\mathcal{L}[\rho_t, s_t, t]$ is known, then the interpolants between observed marginals can be estimated by minimizing the action functional \eqref{eq: minimize_lagrangian_action} with respect to $\mathcal{L}[\rho_t, s_t, t]$\citep{neklyudov2023computational, qian2024conditional, sun2025variational}. However, it is not at all obvious how to pick the correct Lagrangian, since this amounts to choosing a reference process. We therefore consider the Hamiltonian perspective, and seek to learn population mechanics that produce the observed marginals. 

\vspace{-1em}
\paragraph{Mechanics Model.}

We propose a neural mechanics model to learn the population mechanics from a sequence of empirical marginals $\{\hat{p}_{t_i}\}_{i=1}^{M}$ and the initial velocities $\hat{v}_0$. The intuition behind our algorithm comes from \cref{thm:hamiltonian_eqs_wlm}: given the initial conditions, the population dynamics only depend on the potential energy $\mathcal{U}[\rho_t]$ and the damping $\gamma$. We therefore parameterize these variables and estimate them from the observed marginals and the initial velocity $v_0$. The learned mechanical system can then be simulated from any time, and can thus be used for interpolation or forecasting.

First, recall that the acceleration of a particle $x_t^{(j)} \sim p_t$ is given by $\frac{dv_t^{(j)}}{dt} = -\nabla_x \frac{\delta \mathcal{U}[p_t]}{\delta p_t}(x_t^{(j)}) - \gamma v_t^{(j)}$. While the functional derivative is generally intractable, \cref{prop:tractable_accel} (proof in Appendix) shows that it can be bypassed with a tractable characterization.
\begin{restatable}{prop_rep}{TractableAccel}
\label{prop:tractable_accel}
Consider the empirical measure $\hat{p} = \frac{1}{N}\sum_{i=1}^N \delta_{x^{(i)}}$ and let $\Psi(x^{(1)}, \ldots, x^{(N)}) := \mathcal{U}[\hat{p}]$ be its potential energy. Then, for any particle $x^{(j)} \sim \hat{p}$, we have
\begin{align}
    \nabla_{x^{(j)}} \Psi(x^{(1)}, \ldots, x^{(N)}) = \frac{1}{N} \nabla_x \frac{\delta \mathcal{U}[p]}{\delta p }(x^{(j)})\bigg|_{p = \hat{p}}
\end{align}
\end{restatable}

\vspace{-1em}
We may therefore parameterize $\Psi_\theta: \R^{N \times d} \to \R$ with a deep neural network and obtain accelerations from the model via automatic differentiation:
\begin{align}
    a(x_t^{(j)}) = -\nabla_{x_t^{(j)}} \Psi_\theta(x_t^{(1)}, \ldots, x_t^{(N)}) - \gamma v_t^{(j)},
    \label{eq:state_space_autograd}
\end{align}
where we absorb the factor $N$ into the parameterization.

\begin{algorithm}[!t]
\small
\caption{Learning Population Mechanics}
\label{alg:wlm_learning}
\begin{algorithmic}[1]
\State \textbf{Input:} Observed snapshots $\{X_{t_i} = \{x_{t_i}^{(j)}\}_{j=1}^N\}_{i=0}^{M}$, initial velocities $v_0 = \{v_0^{(j)}\}_{j=1}^N$
\State \textbf{Hyperparams:} Divergence $\mathcal{D}$, damping $\gamma \geq 0$, step size $\Delta t$
\State \textbf{Initialize:} Energy model $\Psi_\theta: \mathbb{R}^{N \times d} \to \mathbb{R}$
\For{$\text{iter} = 1$ \textbf{to} max-iterations}
    \State Sample rollout horizon $K \sim \texttt{Unif}[1, M]$
    \State $\{\hat{X}_{t_i}\}_{i=1}^{K} \gets \texttt{P-Mechanics}(\Psi_\theta, \gamma, X_{t_0}, v_0, \Delta t, K)$
    \State $\mathcal{L} \gets \frac{1}{k}\sum_{i=1}^{K} \mathcal{D}(\hat{X}_{t_i}, X_{t_i})$ 
    \State $\theta \gets \texttt{Optimizer}(\theta, \nabla_\theta \mathcal{L})$
    \State \textbf{if} $\gamma$ learnable \textbf{then} $\gamma \gets \texttt{Optimizer}(\gamma, \nabla_\gamma \mathcal{L})$
\EndFor
\State \textbf{Return:} Optimized $\Psi_\theta$ and $\gamma$
\end{algorithmic}

\hrule 
\vspace{0.4em}

\refstepcounter{algorithm}
\label{alg:wlm_sampling}
\begin{algorithmic}[1]
\State \textbf{Input:} $\Psi_\theta$, $\gamma$, $\{x_0^{(j)}\}_{j=1}^N$, $\{v_0^{(j)}\}_{j=1}^N$, $\Delta t$, $K$
\For{$k = 0$ \textbf{to} $K-1$}
    \State $a_k^{(j)} \gets -\nabla_{x_k^{(j)}} \Psi_\theta(x_k^{(1)}, \ldots, x_k^{(N)})$ for all $j$ 
    \State {\footnotesize $x_{k+1}^{(j)}\!,\!v_{k+1}^{(j)} \!\gets\! \texttt{Leapfrog}(x_k^{(j)}\!,\!v_k^{(j)}\!,\!a_k^{(j)}\!,\!\gamma\!,\!\Delta t)$} for all $j$
\EndFor
\State \textbf{Return:} $\{\{x_k^{(j)}\}_{j=1}^N\}_{k=0}^K$
\end{algorithmic}
\end{algorithm}

\vspace{-1em}
\paragraph{Training Algorithm.} 
To train our model, we access accelerations via \eqref{eq:state_space_autograd}, and roll out predicted dynamics from the initial conditions $(\hat{p}_0, \hat{v}_0)$ via the second-order Verlet integrator (Leapfrog), due to its stability for Hamiltonian systems \citep{birdsall2018plasma}. We then set the loss to be the divergence, e.g. Sinkhorn divergence, between the observed marginals, $\{\hat{p}_{t_i}\}_{i=1}^k$, and the model's simulated population. To update the model, we evaluate its gradient with respect to parameters $\theta$ by backpropagating directly through the time-discretized trajectories, i.e. the so-called Discretize-Then-Optimize approach. The pseudo-code for training the model and for simulating the mechanics are given in \cref{alg:wlm_learning} and \cref{alg:wlm_sampling} respectively. We note that temporal features can be added to model time-varying potential energy $\mathcal{U}[\rho_t, t]$, and that the damping parameter can either be fixed, e.g. $\gamma = 0$, or initialized and learned.

\section{Experiments}
\label{sec:experiments}
In \cref{sec:expressivity_wlm}, we proved that WLM describes classical mechanics, quantum mechanics, and gradient flows. We first evaluate on a benchmark dataset of gradient-flow SDEs to demonstrate that \texttt{WLM} can indeed learn any dynamics modeled by gradient flows. We then consider popular real data settings from the literature to evaluate \texttt{WLM}'s performance on non-gradient dynamics. Specifically we consider an ocean vortex dataset \citep{shen2024multi,petrovic2025curly,berlinghieri2025oh} and an embryonic scRNA dataset \citep{tong2020trajectorynet,tong2023simulation,neklyudov2023action,kapusniak2024metric}. Finally, we introduce a novel Boids \citep{reynolds1987flocks} dataset to evaluate inference on emergent dynamics from a population of interacting agents. 

\vspace{-1em}
\paragraph{Experimental Setup.}
We evaluate the ability of methods to \textit{forecast} and \textit{interpolate} by computing the $W_1$ distance between true and predicted marginals.
To evaluate \textit{forecasting} quality, we simulate the learned dynamics from $t=0$, and distinguish the results between seen (train) and unseen (test) marginals. To evaluate \textit{interpolation} quality for an unseen marginal $p_h$, we simulate from the previous observed time $h-1$, as done in \citet{neklyudov2023computational}. This ensures a fairer comparison against flow and SB based models, which are constrained to match the training marginals. To simulate methods that fit an SDE, we use the Euler-Maruyama scheme, and for \texttt{WLM}, we use the leapfrog scheme, as shown in \cref{alg:wlm_sampling}. At inference time, we use $5$ substeps between marginals. We describe the architecture, hyperparameters, and runtimes of \texttt{WLM} in \cref{sec:additional_experiment_appendix}. The code repository is available on GitHub: \url{https://github.com/guanton/WLM}.

\subsection{Gradient-flow Dynamics: Potential-driven SDEs}

To show that \texttt{WLM} can learn gradient flow dynamics, we obtain SDEs such that the drift is a potential landscape $-\nabla V$ from a benchmark dataset \citep{terpin2024learning, persiianov2025learning, guangradient} and we set the diffusivity to be $\sigma^2=1$. For each SDE (the exact potentials are defined in \cref{sec:gradient_flows_appendix}), we simulate $10$ evenly spaced marginals with $1000$ samples each (training data), and reserve the next $10$ marginals for forecasting (test data). We also distinguish between ``paired'' and ``unpaired'' data settings. In the paired setting, we evolve the same $1000$ particles across all times, while in the unpaired setting, we obtain independent samples for each marginal. 

Recall that by \cref{prop:GF_as_WLM}, gradient flows can be characterized as  overdamped systems or as conservative systems with precise initial velocity. We implement both versions for \texttt{WLM}. For the former, we provide the analytic velocity field from the free energy, $v_0(x) = -\nabla(V(x) + \frac{1}{2}\log p_0(x))$ and fix $\gamma = 0$. For the latter, we set $v_0 = 0$ and make friction learnable (initial $\gamma = 1$). 

We consider \textsc{JKOnet}$\ast$ \citep{terpin2024learning}
and \textsc{nn-APPEX} \citep{guangradient} as our gradient flow baselines, due to their performance and ability to 
estimate both the drift $-\nabla V$ and the diffusion $\sigma^2$ from marginals \citep{guangradient}. 

We show results in \cref{tab:sde_results_agg}. Every method succeeds at forecasting in the paired data setting, and \texttt{WLM} (with learnable friction) achieves the best performance in the significantly harder unpaired setting. As observed in \citet[A.3]{persiianov2025learning}, \textsc{JKOnet}$^\ast$ significantly struggles in the unpaired setting due to its reliance on pre-computed couplings. The results therefore show that by making friction learnable, \texttt{WLM} learns gradient flow dynamics as effectively as existing gradient-flow methods. In fact, these results suggest that the learned dissipation by \texttt{WLM} can be used to determine how close population dynamics are to being a gradient flow. We consistently recover $\gamma \ge 500$ in most runs of this experiment (see \cref{fig:friction_comparison_gf_sde} in \cref{sec:gradient_flows_appendix}).

\begin{table} 
\centering
\caption{Gradient flow SDEs: We report the average $W_1$ distances between true marginals and learned marginals rolled out from \textsc{JKOnet}*, \textsc{nn-APPEX}, and WLM. We consider both paired and unpaired settings, and we distinguish between train (first $10$ marginals) and forecast (next $10$ marginals). Results are also averaged across all 5 SDEs. The full breakdown per SDE is in \cref{tab:full_gf_sde_results}.}
\label{tab:sde_results_agg}
\resizebox{\columnwidth}{!}{%
\begin{tabular}{lcccc}
\toprule
 & \multicolumn{2}{c}{\textbf{Paired}} & \multicolumn{2}{c}{\textbf{Unpaired}} \\
\cmidrule(lr){2-3} \cmidrule(lr){4-5}
\textbf{Method} & Train $W_1$  & Forecast $W_1$ & Train $W_1$ & Forecast $W_1$  \\
\midrule
\textsc{JKOnet}* & $0.085 \pm 0.007$ & $0.193 \pm 0.020$ & $0.236 \pm 0.040$ & $1.618 \pm 0.261$ \\
\textsc{nn-APPEX}   & $0.080 \pm 0.006$ & \textbf{0.131} $\pm$ \textbf{0.006} & $0.102 \pm 0.008$ & $0.260 \pm 0.025$ \\
\texttt{WLM} (learnable friction)               & \textbf{0.062} $\pm$ \textbf{0.004} & $0.137 \pm 0.012$ & \textbf{0.068} $\pm$ \textbf{0.004} & \textbf{0.246} $\pm$ \textbf{0.026} \\
\texttt{WLM} (0 friction)               & $0.119 \pm 0.012$ & $0.255 \pm 0.037$ & $0.119 \pm 0.011$ & $0.346 \pm 0.045$ \\
\bottomrule
\end{tabular}%
}
\end{table}

\subsection{Curling Dynamics: Ocean vortex data}
We test \texttt{WLM}'s ability to learn the curling dynamics of particles within an ocean vortex, whose velocity field was measured in the Gulf of Mexico. We consider the interpolation experiment from previous works \citep{shen2024multi, petrovic2025curly}: given that $111$ particles spiral out from the eye of the vortex, and are observed at $5$ uniformly spaced times (train on $t_1, t_3, t_5, t_7,t_9)$, the task is to interpolate $4$ intermediary marginals (test on $t_2, t_4, t_6, t_8$). We also consider the same forecasting experiment from \citet{berlinghieri2025oh}, which considers a larger vortex from the same dataset. There, the task is to predict the last marginal, after training on $10$ previous marginals of $400$ samples each. For the interpolation task, we add time features to model a time-varying potential energy $\mathcal{U}[\rho_t,t]$. To perform the forecasting task, we do not consider time features. For more implementation details, see \cref{sec:ocean_currents_appendix}.

We report results for interpolation and forecasting in \cref{tab:oceans_metrics}, noting that only \texttt{WLM}, \textsc{SBIRR}, and \textsc{Snap-MMD} are capable of forecasting. \texttt{WLM} achieves the best interpolations among methods that do not use a reference process, and the second best overall performance after \textsc{Curly-FM}, which trains a flow matching method based on a `curly' reference process constructed from velocity fields at all training marginals. In contrast, \texttt{WLM} learns a mechanics model directly from the observed marginals and the initial velocity. We visualize \texttt{WLM}'s ability to predict unseen marginals along the vortex's curvature in \cref{fig:oceans}. Conversely, gradient-flow and OT-based methods struggle to interpolate along the vortex, with the notable exception of \textsc{OT-MFM}, which leverages the data manifold structure. We also emphasize that by learning the mechanics, \texttt{WLM} is able to forecast, whereas flow matching methods like \textsc{Curly-FM} can only interpolate. While \textsc{SBIRR} \citep{shen2024multi} and \textsc{Snap-MMD} \citep{berlinghieri2025oh} produce reasonable forecasts by using reference SDEs based on vortex dynamics \citep[D.5.2]{shen2024multi} \citep[D.9.2]{berlinghieri2025oh}, \texttt{WLM} produces more accurate forecasts by directly learning second-order mechanics. 

\begin{table*}[!t]
\centering
\caption{$W_1$ distance (averaged over 3 seeds) for interpolating and forecasting on the Gulf of Mexico dataset. We underline best results for methods that use a reference process from domain knowledge and bold best results for methods that do not.}
\label{tab:oceans_metrics}
\footnotesize
\renewcommand{\arraystretch}{0.85}
\setlength{\tabcolsep}{3pt}

\medmuskip=0mu

\begin{tabular}{lcccccc}
\toprule
& \multicolumn{4}{c}{Interpolation (small vortex)} & \multicolumn{2}{c}{Forecast (big vortex)} \\
\cmidrule(lr){2-5} \cmidrule(lr){6-7}
Method & $t_2$ & $t_4$ & $t_6$ & $t_8$ & $t_{11}$ (from $t_0$) & $t_{11}$ (from $t_{10}$) \\
\midrule
\multicolumn{7}{l}{\textbf{Uses reference process from domain knowledge}} \\
\textsc{Curly-FM}* {\scriptsize \citep{petrovic2025curly}} & $\underline{0.019 \pm 0.003}$ & $\underline{0.045 \pm 0.005}$ & $\underline{0.027 \pm 0.001}$ & $\underline{0.030 \pm 0.006}$ & -- & -- \\
\textsc{SBIRR} {\scriptsize \citep{shen2024multi}} & $0.073 \pm 0.035$ & $0.087 \pm 0.033$ & $0.062 \pm 0.010$ & $0.082 \pm 0.022$ & $1.062 \pm 0.111$ & $0.567 \pm 0.014$ \\
\textsc{Snap-MMD} {\scriptsize \citep{berlinghieri2025oh}} & $0.051 \pm 0.003$ & $0.067 \pm 0.002$ & $0.131 \pm 0.005$ & $0.056 \pm 0.001$ & $\underline{0.896 \pm 0.141}$ & $\underline{0.473 \pm 0.013}$ \\
\midrule
\multicolumn{7}{l}{\textbf{No reference process from domain knowledge}} \\
\textsc{AM} {\scriptsize \citep{neklyudov2023action}} & $0.358 \pm 0.009$ & $0.353 \pm 0.020$ & $0.447 \pm 0.013$ & $0.346 \pm 0.016$ & -- & -- \\
\textsc{UAM} {\scriptsize \citep{neklyudov2023action}} & $0.291 \pm 0.008$ & $0.378 \pm 0.021$ & $0.514 \pm 0.010$ & $0.349 \pm 0.009$ & -- & -- \\
\textsc{sAM} {\scriptsize \citep{neklyudov2023action}} & $0.358 \pm 0.009$ & $0.354 \pm 0.010$ & $0.451 \pm 0.016$ & $0.350 \pm 0.009$ & -- & -- \\
\textsc{DICE} {\scriptsize \citep{blickhan2025dice}} & $0.369 \pm 0.019$ & $0.252 \pm 0.020$ & $0.240 \pm 0.011$ & $0.115 \pm 0.009$ & -- & -- \\
\textsc{Vanilla-SB} {\scriptsize \citep{shen2024multi}} & $0.112 \pm 0.061$ & $0.178 \pm 0.073$ & $0.262 \pm 0.043$ & $0.305 \pm 0.032$ & -- & -- \\
\textsc{OT-CFM}* {\scriptsize \citep{tong2023improving}} & $0.148 \pm 0.004$ & $0.227 \pm 0.008$ & $0.191 \pm 0.012$ & $0.250 \pm 0.018$ & -- & -- \\
\textsc{OT-MFM}* {\scriptsize \citep{kapusniak2024metric}} & $0.107 \pm 0.014$ & $0.056 \pm 0.014$ & $\mathbf{0.052 \pm 0.011}$ & $0.070 \pm 0.021$ & -- & -- \\
\texttt{WLM} (Ours) & $\mathbf{0.039 \pm 0.002}$ & $\mathbf{0.028 \pm 0.000}$ & $0.083 \pm 0.002$ & $\mathbf{0.050 \pm 0.004}$ & $\mathbf{0.180 \pm 0.038}$ & $\mathbf{0.065 \pm 0.005}$ \\
\bottomrule
\multicolumn{7}{r}{\scriptsize *Numbers taken from \citet{petrovic2025curly}} \\
\end{tabular}
\end{table*}

\newcommand{\orangecross}{{\color{orange}$\times$}}

\begin{figure*}[t]
    \centering
    \begin{subfigure}[b]{0.28\textwidth}
        \centering
        \includegraphics[height=3.8cm]{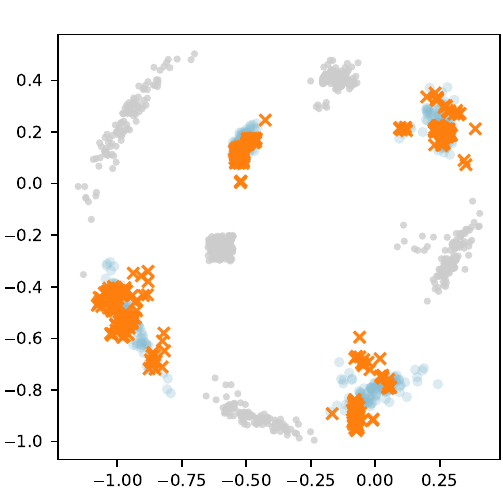}%
        \caption{Ocean vortex}
        \label{fig:oceans}
    \end{subfigure}%
    \hfill
    \begin{subfigure}[b]{0.70\textwidth}
        \centering
        \includegraphics[height=3.8cm]{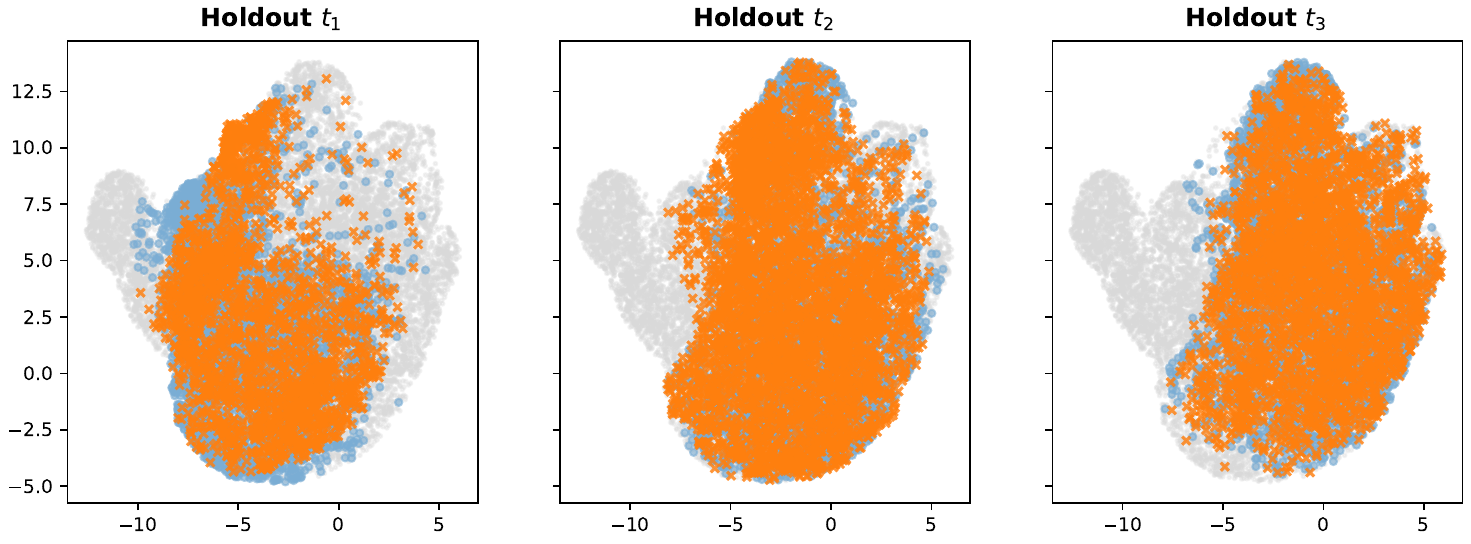}%
        \caption{Embryoid body}
        \label{fig:eb_umap}
    \end{subfigure}
    \caption{\texttt{WLM}'s predictions (\orangecross) for unseen
    interpolants (\bluecircle) are visualized for the
    (a) Ocean vortex (in spatial coordinates) and (b)
    Embryoid body (in UMAP coordinates) datasets. Marginals that are used to train \texttt{WLM}'s mechanics model are plotted in
    \textcolor{gray}{gray}.}
    \label{fig:oceans_vs_eb}
\end{figure*}

\begin{table}[!t]
\centering
\caption{Average $W_1$ distance for leave-one-out interpolation on $5$-dim PCA representation of EB scRNA dataset. Results are averaged over $5$ runs.}
\label{tab:results_eb}
\footnotesize
\renewcommand{\arraystretch}{0.85}
\setlength{\tabcolsep}{2pt}
\begin{tabular}{l r}
\toprule
Method & $W_1$ \\
\midrule
\textsc{Reg. CNF}* \citep{finlay2020train} & $0.825 \pm 0.429$ \\
\textsc{TrajectoryNet}*\citep{tong2020trajectorynet} & $0.848$ \\
\textsc{NLSB}* \citep{koshizuka2022neural} & $0.970$ \\
\textsc{DSBM}* \citep{shi2024diffusion} & $1.775 \pm 0.429$ \\
\textsc{DSB}* \citep{de2021diffusion} & $0.862 \pm 0.023$ \\
\textsc{SF$^2$M-Sink}* \citep{tong2023simulation} & $1.198 \pm 0.342$ \\
\textsc{SF$^2$M-Geo}* \citep{tong2023simulation} & $0.879 \pm 0.148$ \\
\textsc{SF$^2$M-Exact}* \citep{tong2023simulation} & $0.793 \pm 0.066$ \\
\textsc{OT-CFM}* \citep{tong2023improving} & $0.790 \pm 0.068$ \\
\textsc{I-CFM}* \citep{tong2023improving} & $0.872 \pm 0.087$ \\
\textsc{SB-CFM}* \citep{tong2023improving} & $1.221 \pm 0.380$ \\
\textsc{WLF-UOT}* \citep{neklyudov2023computational} & $0.738 \pm 0.014$ \\
\textsc{WLF-SB}* \citep{neklyudov2023computational} & $0.746 \pm 0.016$ \\
\textsc{WLF-OT}* \citep{neklyudov2023computational} & $0.742 \pm 0.012$ \\
\textsc{I-MFM}*\citep{kapusniak2024metric} & $0.822 \pm 0.042$ \\
\textsc{OT-MFM}*\citep{kapusniak2024metric} & $0.713 \pm 0.039$ \\
\textsc{AM} \citep{neklyudov2023action} & $0.944 \pm 0.003$ \\
\textsc{UAM} \citep{neklyudov2023action} & $0.924 \pm 0.004$ \\
\textsc{sAM} \citep{neklyudov2023action} & $0.939 \pm 0.003$ \\
\textsc{DICE} \citep{blickhan2025dice} & $0.826 \pm 0.002$ \\
\texttt{WLM} (Ours) & $\mathbf{0.704 \pm 0.021}$ \\
\bottomrule
\multicolumn{2}{r}{\scriptsize *Numbers taken from \citet{kapusniak2024metric}} \\
\end{tabular}
\end{table}

\begin{figure*}[t]
    \centering
    \begin{tikzpicture}
        \node[anchor=south west, inner sep=0] (image) at (0,0) {
            \includegraphics[width=0.965\linewidth]{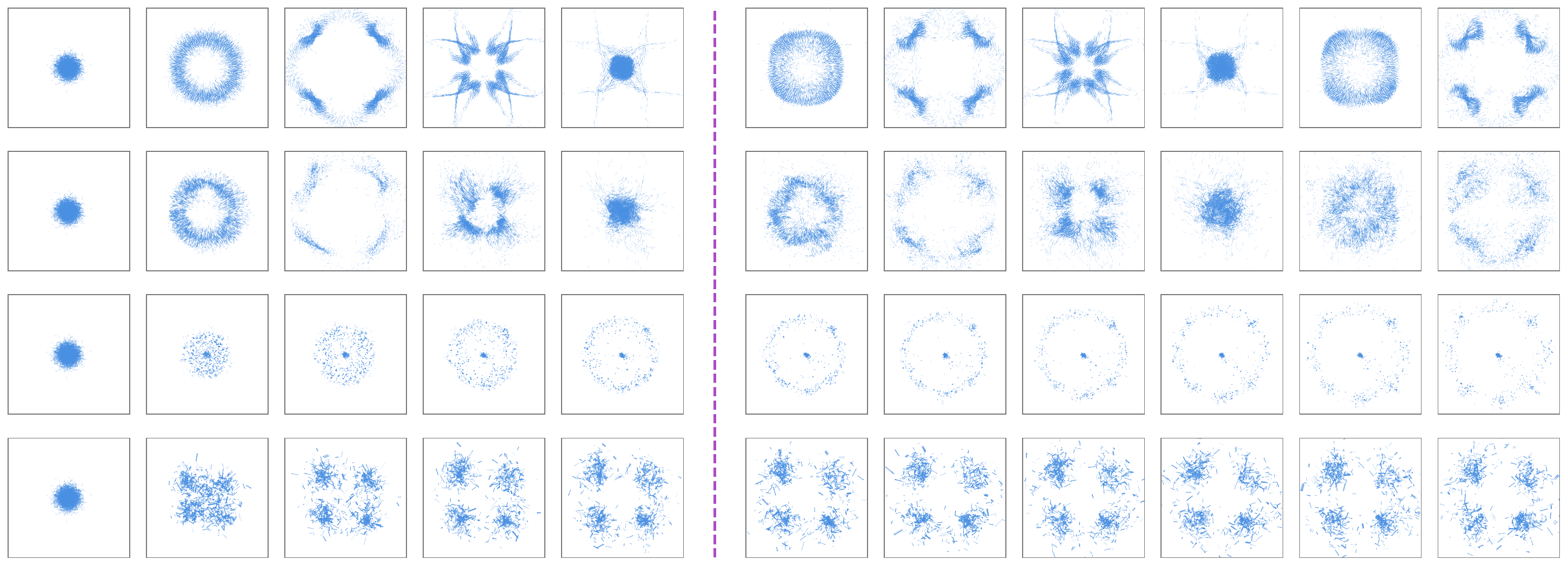} 
        };

        \begin{scope}[x={(image.south east)}, y={(image.north west)}]

            
            \def\rowY{0.02}   
            \def\colY{1.01}   
            \def\labelSep{-0.015} 

            
            \node[rotate=90, anchor=south, font=\tiny] at (\labelSep, 0.875) {Ground Truth};
            
            \node[rotate=90, anchor=south, font=\tiny] at (\labelSep, 0.625) 
                {\scalebox{1.3}{\texttt{WLM}}};
            
            \node[rotate=90, anchor=south, font=\tiny] at (\labelSep, 0.375) 
                {\scalebox{1}{\textsc{JKOnet}$^\ast$}};
            
            \node[rotate=90, anchor=south, font=\tiny] at (\labelSep, 0.125) 
                {\scalebox{1}{\textsc{nn-APPEX}}};

            \foreach \x/\t in {
                0.044/0.0, 
                0.133/5.0, 
                0.221/10.0, 
                0.310/15.0, 
                0.398/20.0
            } {
                \node[anchor=south, font=\tiny] at (\x, \colY) {$t=\t$};
            }

            \foreach \x/\t in {
                0.513/25.0, 
                0.602/30.0, 
                0.690/35.0, 
                0.779/40.0, 
                0.867/45.0, 
                0.956/50.0
            } {
                \node[anchor=south, font=\tiny, color=forecastcolour] 
                    at (\x, \colY) {$t=\t$};
            }

        \end{scope}
    \end{tikzpicture}
    
    \caption{
        \textbf{Learning Boids:} 
        We plot the ground truth Boids (first row) against the population dynamics learned by our method \texttt{WLM} (second row), and the gradient flow baselines \textsc{JKOnet}$^\ast$ (third row) and \textsc{nn-APPEX} (fourth row). All methods were trained on $50$ marginals of size $1000$ within $t \in [0,24.5]$. The rollouts of each method are simulated from time $0$ (with $5000$ samples drawn) until the additional \textcolor{forecastcolour}{forecast horizon} $[25.0, 50.0]$ (indicated by {purple time labels}). We include additional figures of how the model learned by \texttt{WLM} generalizes to Boids dynamics from unseen initial populations in \cref{fig:app_boids_generalization}.
    }
    \label{fig:boids_comparison}
\end{figure*}

\subsection{Cell Dynamics: Embryonic scRNA Data}
Inferring the developmental trajectories of evolving cells is a notoriously challenging problem, due to its biological complexity, as well as inherent challenges in the sc-RNA measurement process, which is necessarily destructive \citep{trapnell2014dynamics}. To test \texttt{WLM}'s ability to learn the evolution of cell populations, we consider the leave-one-out interpolation task for each of the three intermediate time marginals of the embryoid body (EB) dataset \citet{moon2019visualizing}. The EB dataset measures the differentiation of human embryonic stem cells over the course of $27$ days, such that the data is split into $5$ marginals of sizes $2381$, $4163$, $3278$, $3665$ and $3332$. By following the standard preprocessing of whitening the $5$-dimensional PCA data \citep{tong2020trajectorynet}, we compare \texttt{WLM} against a comprehensive list of methods evaluated on the same experiment \citep{kapusniak2024metric}. 

We implement \texttt{WLM} without including time features in the potential energy $\mathcal{U}[\rho_t]$ and report the results in \cref{tab:results_eb}. Even without leveraging explicit temporal information, \texttt{WLM} outperforms existing state-of-the-art methods. We summarize a few key insights. First, since cell dynamics remain unknown, a common approach has been to guess different priors and hope to produce developmental trajectories that have a reasonable curvature. Indeed, many of the methods in \cref{tab:results_eb} use different optimal transport or Lagrangian priors, which can significantly influence inference quality. In contrast, \texttt{WLM} directly learns population mechanics from the data, without specifying a reference process. Furthermore, \texttt{WLM}'s learned time-independent potential energy and friction can be interpreted in order to gain insight into the nature of the dynamics. In this experiment, we recover dissipative dynamics from moderate to high friction (see \cref{fig:eb_friction_holdout}), which is consistent with the fact that 
mathematical biologists commonly use gradient flows to model cell dynamics \citep{waddington2015animals, weinreb2018fundamental, lavenant2021towards}. Finally, \texttt{WLM}'s result on this dataset can likely be improved, since we did not use a validation set for early-stopping, as done in \textsc{OT-MFM} \citep{kapusniak2024metric}. 

Since the EB dataset contains the most samples per marginal, we also run an additional ablation to investigate the relationship between \texttt{WLM}'s performance ($W_1$ on holdouts) and runtime when populations are subsampled as mini-batches during training. We report results in \cref{fig:w1_vs_runtime} and interestingly note that in addition to decreasing runtime, mini-batching induces moderate performance gains for this experiment.

\begin{table}[!t]
\centering
\caption{Boids: Average $W_1$ distances between true and learned marginals. We distinguish between train times (first $50$) and forecast (next $50$).}
\label{tab:results_boids}

\footnotesize          
\renewcommand{\arraystretch}{0.85} 
\setlength{\tabcolsep}{4pt} 

\begin{tabular}{lcc}
\toprule
\textbf{Method} & Train $W_1$ ($\pm$SE) & Forecast $W_1$ ($\pm$SE) \\
\midrule
\textsc{JKOnet}* & $3.084 \pm 0.229$ & $3.174 \pm 0.084$ \\
\textsc{JKOnet}* (time-var.) & $4.400 \pm 0.229$ & $10.235 \pm 0.341$ \\
\textsc{nn-APPEX}   & $2.467 \pm 0.193$ & $3.078 \pm 0.199$ \\
\texttt{WLM} (ours)               & \textbf{0.496} $\pm$ \textbf{0.024} & \textbf{1.309} $\pm$ \textbf{0.034} \\
\bottomrule
\end{tabular}
\end{table}

\subsection{Emergent dynamics: Boids}
Murmuration is a swarming behaviour observed in large flocks of birds, whose collective flight produces intricate patterns in the sky. Boids is an influential model that was developed by \citet{reynolds1987flocks}. Indeed, Boids has been used to simulate swarms in high-budget movies \footnote{e.g. bat swarms and penguin flocks in Batman Returns (1992)} \citep{bajec2009organized} and is also used to model cognition \citep{rosas2020reconciling, lawrence2024aggregate} and automated vehicles \citep{saska2014swarms, knievel2023boids}. The algorithm simulates interacting agents that obey three rules:
\begin{enumerate}[noitemsep,nosep]
    \item \textbf{\textit{Separation}}: steer to avoid colliding with flockmates,
    \item \textbf{\textit{Alignment}}: steer towards average heading of flock,
    \item \textbf{\textit{Cohesion}}: steer towards the flock's centre of mass.
\end{enumerate}
While there are many approaches for simulating Boids \citep{schoenholz2020jax, hartman2006autonomous}, analytic solutions are intractable due to the complexity of the emergent dynamics.
 Given the absence of a reference process, we consider the gradient flow methods \textsc{JKOnet}$^\ast$ and \textsc{nn-APPEX} as baselines. To model mechanics that are completely determined by the population's potential energy, we implement \texttt{WLM} with $\gamma = 0$ and do not model time-dependence. All methods are trained on $50$ evenly spaced marginals with $1000$ Boids per time. We then reserve the next $50$ marginals as the test set to evaluate the ability of all methods to forecast the dynamics. 

We present numerical results in \cref{tab:results_boids} and qualitative results in \cref{fig:boids_comparison}. From these, we see that \textsc{JKOnet}$\ast$ and \textsc{nn-APPEX} cannot even fit the observed dynamics, as they only simulate slowly diffusing agents. This corroborates the expectation that Boids is far from being a gradient flow, due to its often periodic dynamics. In contrast, \texttt{WLM} learns population mechanics that cohere to the Boids' cohesion and repulsion agent-based rules. While the \texttt{WLM} rollouts lack some finer details, \texttt{WLM} simulates flocking behaviour, which match each of the outward and inward migrations of the Boids population, even when forecasting the population $50$ marginals beyond its training data. \texttt{WLM}'s learned mechanics also generalize to unseen populations, as visualized in \cref{fig:app_boids_generalization}, and in our interactive notebook \footnote{\scriptsize\url{https://colab.research.google.com/drive/1R4q5pysHaYOMIu77vT6JID7yCLYG3t4z?usp=sharing}}.

As an additional experiment, we test \texttt{WLM}'s robustness for learning Boids dynamics when the initial velocity is misspecified. Recall that \texttt{WLM} is based on our derivations in \cref{sec: equations_of_motion}, which state that the population mechanics are determined by the potential energy and the true initial velocities. In this experiment, we consider the cases where the initial velocity is completely misspecified as $0$, and when the initial velocity is estimated from the data, using DICE \citep{blickhan2025dice}. While access to the true initial velocity helps the model, accurate forecasts are still achieved when the velocity is misspecified (see \cref{tab:boids_v0_experiment}).

\section{Related Work}
\label{sec:related_work}
\paragraph{Wasserstein Lagrangian and Hamiltonian Flows}
The mathematical literature has developed Lagrangian and Hamiltonian formulations for modeling conservative dynamics in the Wasserstein space \citep{ambrosio2008hamiltonian, chow2020wasserstein}. \citet{neklyudov2023computational} was the first work to use Wasserstein Lagrangian flows as a model for inferring population dynamics given observational data. The main practical limitation of this work and related methods \citep{qian2024conditional,sun2025variational} is that the Lagrangian formulation requires a reference process in order to obtain a well-defined action minimization \eqref{eq: minimize_lagrangian_action} for inferring interpolants. In our work, we derive the Hamiltonian mechanics from \emph{damped Wasserstein Lagrangian flows}. This enables our method \texttt{WLM} can interpolate and forecast, without pre-specifying the Lagrangian. The incorporation of damping importantly allows \texttt{WLM} to learn dissipative population dynamics, which provide a stable representation of Wasserstein gradient flows in the overdamped limit.

\vspace{-1em}
\paragraph{Wasserstein Gradient Flows.} 
Gradient flows are the predominant model for population dynamics \citep{hashimoto2016learning, weinreb2018fundamental, lavenant2021towards, bunne2022proximal, neklyudov2023action, terpin2024learning, guangradient}. In this work, we build on mathematical ideas from Wasserstein gradient flows \citep{jordan1998variational,ambrosio2008gradient} in order to develop WLM as a richer class of second-order dynamics on the same space. We then demonstrate that \texttt{WLM} learns gradient flows as effectively as state-of-the-art gradient flow methods, while also effectively learning non-gradient dynamics, like Boids.
\vspace{-1em}
\paragraph{Reference-based Population Dynamics.} 
Many methods infer population dynamics given a reference process \citep{schiebinger2019optimal, bunne2023schrodinger,berlinghieri2025oh,petrovic2025curly}, or family of reference processes \citep{shen2024multi,zhang2024joint,guan2024identifying, guangradient}. Reference-based inference is often done to overcome the limitations of gradient flows \citep{shen2024multi, zhang2024joint, guan2024identifying,petrovic2025curly, berlinghieri2025oh}. Our work also addresses the limitations of gradient flows, but \texttt{WLM} is not reference-based, and is therefore less susceptible to model misspecification.
\vspace{-1em}
\paragraph{Flow and Action Matching.}
Flow matching \citep{lipman2022flow} and action matching \citep{neklyudov2023action} have been proposed to learn continuous interpolations from a finite set of time marginals, and have each inspired methods tailored to scientific inference \citep{atanackovic2024meta, kapusniak2024metric, blickhan2025dice}. However, unlike \texttt{WLM}, flow matching and action matching methods are unable to forecast, since these models are only trained to interpolate within the observation period. 


\section{Conclusion}
\label{sec:conclusion}
In this work, we proposed and advanced a new direction for inferring population dynamics, by emphasizing the \emph{population mechanics} as an intrinsic and learnable property. We developed a comprehensive theory for \emph{Wasserstein Lagrangian Mechanics} by considering a population-level principle of least action with damping, which produces second-order Hamiltonian mechanics that encompass classical mechanics, quantum mechanics, and gradient flows. We then leverage these insights to develop \texttt{WLM}, a powerful inference method that outperforms state-of-the-art methods for interpolating and forecasting unseen marginals on real and synthetic datasets from the literature. \texttt{WLM} also learns the emergent dynamics from interacting Boids, demonstrating exciting new capabilities from learning the mechanics.

\paragraph{Limitations and Directions for Future Work.}
While we believe that this work makes important theoretical and practical advances, the learning of population mechanics is a nascent field. We outline limitations from our paper, which may provide interesting opportunities for future work.

There are many open questions about the expressivity and identifiability of WLM systems. While WLM can describe classical mechanics, quantum mechanics, and gradient flows, the population dynamics of other systems may not obey a Wasserstein least action principle. With respect to identifiability, we have already seen that gradient flows admit multiple WLM formulations in \cref{prop:GF_as_WLM}. We conjecture that under certain assumptions, like $\gamma = 0$, non-degeneracy, and non-equilibrium, a wide class of WLM systems can be uniquely identified from continuous population dynamics. Furthermore, without imposing additional assumptions, various individual trajectories produce the same population dynamics. In this work, we model trajectories using the ODE from the canonical continuity equation. Interesting future work remains for modeling other trajectories of interest, such as solutions to the Navier-Stokes equations or solutions to the Schr\"odinger equation.

For inference, the largest practical barrier is that \texttt{WLM} simulates dynamics in order to train the mechanics model. Given sufficiently dense data, a simulation free approach may be feasible by first learning the vector field $\nabla s_t$, and then extracting the Hamiltonian mechanics. Our initial attempts were unsuccessful, as any errors incurred by the estimation are inherited downstream. 

\newpage
\paragraph{Broader Impact.}
This work is primarily a theoretical and methodological contribution towards scientific inference, and poses little risk for social harm. Better prediction of population dynamics could potentially be used for malicious development, but we do not see this as a significant risk at this time, as the theory primarily pertains to cellular and molecular systems.

\section*{Acknowledgments}
VG was supported by the PGS-D scholarship, funded by the Natural Sciences and Engineering Research Council of Canada, and was affiliated with Mila as a visiting researcher. The research was enabled in part by computational resources provided by the Digital Research Alliance of Canada (\url{https://alliancecan.ca}) and Mila (\url{https://mila.quebec}). In addition, KN was supported by IVADO and Institut Courtois. LA was in-part supported by the Eric and Wendy Schmidt Center at the Broad Institute of MIT and Harvard, and by the NSERC Postdoctoral Fellowship.
This project was undertaken thanks to funding from IVADO and the Canada First Research Excellence Fund.

\bibliography{references}
\bibliographystyle{icml2026}


\clearpage
\onecolumn
\appendix

\section{Additional Proofs}
\label{sec:app_proofs}

\begin{theorem}[Hamiltonian equations of motion]
\label{thm:hamiltonian_main_thm_appendix}
Let $(p_t)_{0 \le t \le 1}$ minimize action with respect to the damped Wasserstein Lagrangian 
\begin{align}
    \mathcal{L}[\rho_t, s_t,t] = e^{\gamma t} \left(\frac{1}{2}\int \|\nabla s_t\|^2 \rho_t dx - \mathcal{U}[\rho_t]\right)
\end{align}
with damping $\gamma \ge 0$. Then, $(p_t)_{0 \le t \le 1}$ obeys the continuity equation \eqref{eq: continuity_eq}, such that the driving potential $s_t$ obeys 
\begin{align}
    \dot{s}_t(x)&= -\frac{1}{2}\|\nabla s_t(x)\|^2 - \frac{\delta \mathcal{U}[\rho_t]}{\delta \rho_t}(x) - \gamma s_t(x).
\end{align}
The marginals $(p_t)_{0 \le t \le 1}$ can therefore be produced by sampling trajectories, obeying the second order mechanics:
\begin{mdframed}[
    hidealllines=true, 
    backgroundcolor=gray!8, 
    nobreak=true,
    innertopmargin=0pt,    
    innerbottommargin=10pt  
]
\begin{align}
    \frac{d}{dt} x_t &= v_t \coloneq \nabla s_t(x_t), & x_0 \sim p_0 \\
    \frac{d}{dt} v_t &= -\nabla \frac{\delta \mathcal{U}[\rho_t]}{\delta \rho_t}(x_t) - \gamma v_t, & v_0 = \nabla s_0(x_0)
\end{align}
\end{mdframed}
\end{theorem}
\begin{proof}
The action functional is given by
\begin{align}
\mathcal{S}[\rho_t, s_t] &=  \int_0^1 e^{\gamma t}\left(\frac{1}{2}\int \|\nabla s_t\|^2 \rho_t dx - \mathcal{U}[\rho_t]\right) dt.
\label{eq:damped_action_functional}
\end{align}
First, we note that by the continuity equation \eqref{eq: continuity_eq}, the term
$\int_0^1 e^{\gamma t} \int s_t \bigl(\dot{\rho}_t + \nabla \cdot (\rho_t \nabla s_t) \bigr)dx dt$ is identically $0$, so we can add it to the damped action functional \eqref{eq:damped_action_functional}. Integrating by parts allows us to rewrite $\mathcal{S}[\rho_t, s_t]$ as
\begin{align*}
     \mathcal{S}[\rho_t, s_t]=\int_0^1 e^{\gamma t}\left( -\frac12 \int \|\nabla s_t\|^2 \rho_t dx - \mathcal{U}[\rho_t] + \int  s_t \dot{\rho}_t dx\right) dt.
\end{align*}
Then, by the principle of least action, we have the optimality conditions
\eqref{eq:first_order_optimality}. While $\frac{\delta \mathcal{S}[\rho_t, s_t]}{\delta s_t}[h_t]=0$ simply recovers the continuity equation \eqref{eq: continuity_eq}, we can use $\frac{\delta \mathcal{S}[\rho_t, s_t]}{\delta \rho_t}[h_t]=0$ to derive the evolution of $s_t$. Integrating by parts yields
\begin{align*}
    0 = \int_0^1e^{\gamma t}  \int \left(-\frac12 h_t \|\nabla s_t\|^2  - h_t \frac{\delta \mathcal{U}[\rho_t]}{\delta \rho_t}+ s_t \dot{h}_t\right)  dx dt\\
    = \int_0^1  e^{\gamma t}  \int  h_t\underbrace{\left(-\frac12 \|\nabla s_t\|^2 - \frac{\delta \mathcal{U}[\rho_t]}{\delta \rho_t} - \dot{s}_t - \gamma s_t \right)}_{=0}dxdt
\end{align*}
We therefore obtain the evolution equation of the driving potential \eqref{eq:potential_evolution}, since $e^{\gamma t} > 0$ and the above holds for all test functions $h_t$.
Then, to determine the induced trajectory $(x_t)_{0 \le t \le 1}$ for a particle $x_0 \sim p_0$, we set $v_t(x)=\nabla s_t(x)$ and compute the total derivative of the particle's time-varying velocity, $\dot x_t=v_t(x_t)$:
\begin{align}
\frac{d}{dt}v_t(x_t)
&= \bigl(\partial_t v_t + (v_t\!\cdot\nabla)v_t\bigr)(x_t),
\label{eq:expanded_material_derivative}
\end{align}
where we clarify that the operator $v_t \cdot \nabla = \sum_{i=1}^{d}v_t^{(i)} \frac{\partial}{\partial x^{(i)}}$ and hence, $v_t \cdot \nabla v_t$ is a vector with components
\begin{align*}
    [v_t \cdot \nabla v]_j= \sum_{i=1}^{d}v_t^{(i)} \frac{\partial v_t^{(j)}}{\partial x^{(i)}}
\end{align*}
The first term in \eqref{eq:expanded_material_derivative} is computed by taking the gradient on both sides of the evolution equation of $s_t$ \eqref{eq:potential_evolution}:
\[
\partial_t v_t 
= -\nabla\!\Bigl(\frac{\delta\mathcal U[\rho_t]}{\delta\rho_t}\Bigr) - \gamma v_t - \nabla (\frac12 \|v_t\|^2).
\]
Then, the second term simplifies to $(v_t\!\cdot\nabla)v_t = \nabla (\frac12 \|v_t\|^2)$, since $v_t=\nabla s_t$ is conservative. It follows that
\begin{align*}
    \frac{d}{dt}v_t &= -\nabla \frac{\delta \mathcal{U}[\rho_t]}{\delta \rho_t}(x_t) - \gamma v_t, \quad v_0 = \nabla s_0(x_0)
\end{align*}
\end{proof}

Intuitively, the potential energy determines the evolution of the velocity field, which determines the population dynamics. These population dynamics are attained by aggregating paths from the corresponding ODE.


\begin{prop}[Wasserstein Lagrangian of a gradient flow]
Let $\gamma = 0$. Then, $(\rho_t)_{0 \le t \le 1}$ is a gradient flow of $\mathcal{F}[\rho_t]$ if and only if it minimizes the Wasserstein Lagrangian action 
\begin{align}
    \mathcal{S}[\rho_t,s_t] &= \int_0^1 \left\{ \mathcal{K}[\rho_t,s_t] - \mathcal{U}[\rho_t] \right\}dt\\
    &\coloneq\int_0^1 \left\{ \frac12 \int \| \nabla s_t\|^2 \rho_t \ dx  + \frac 12\int \| \nabla \frac{\delta}{\delta \rho_t} \mathcal{F}[\rho_t]\|^2 \rho_t \ dx \right\}dt,
\end{align}
and additionally has initial velocity $v_0 \propto -\nabla \frac{\delta}{\delta \rho_0}\mathcal{F}[\rho_0]$.
\label{prop:WGF_as_WLM}
\end{prop}

\begin{proof}
Recall that population dynamics $(\rho_t)_{0 \le t \le 1}$ are a Wasserstein gradient flow of $\mathcal{F}$ if and only if $\nabla s_t \propto -\nabla \frac{\delta}{\delta \rho_t}\mathcal{F}(\rho_t)$ for $\rho_t$ a.e. $(x,t) \in \R^d \times [0,1]$ \citep{santambrogio2017euclidean}. We first prove that a necessary condition is that the dynamics minimize the Wasserstein Lagrangian action given by the $W_2$ kinetic energy \eqref{eq:canonical_kinetic_energy} and the potential energy $\mathcal{U}[\rho_t] = -\frac 12 \int  \|\nabla \frac{\delta}{\delta \rho_t} \mathcal{F}[\rho_t] \|^2\rho_tdx$. 

Let $\mathcal{F}$ be a sufficiently regular free energy functional so that $\frac{\delta}{\delta \rho_t}\mathcal{F}(\rho_t)$ exists and let $\rho_0$ and $\rho_1$ be two observed marginals. By definition, if $(\rho_t)_{0 \le t \le 1}$ is a gradient flow for $\mathcal{F}$, then it minimizes the non-negative objective
\begin{align}
    \mathcal{J}[\rho_t, s_t] = \frac12\int_0^1 \int \| \nabla s_t + \nabla \frac{\delta}{\delta \rho_t} \mathcal{F}[\rho_t]\|^2 \rho_t \ dx \  dt,
    \label{eq:gradient_flow_cost}
\end{align}
Thus, given fixed $\rho_0=p_0$ and $\rho_1 = p_1$, we seek to minimize \eqref{eq:gradient_flow_cost}. However, we may rewrite
\begin{align*}
    \mathcal{J}[\rho_t, s_t] &= \int_0^1 \left\{ \frac12 \int \| \nabla s_t\|^2 \rho_t \ dx  + \frac 12\int \| \nabla \frac{\delta}{\delta \rho_t} \mathcal{F}[\rho_t]\|^2 \rho_t \ dx  \right\}dt +\int_0^1 \int \nabla \frac{\delta}{\delta \rho_t} \mathcal{F}[\rho_t] \cdot \nabla s_t \rho_t \ dx \ dt \\
    &= \int_0^1 \left\{ \frac12 \int \| \nabla s_t\|^2 \rho_t \ dx  + \frac 12\int \| \nabla \frac{\delta}{\delta \rho_t} \mathcal{F}[\rho_t]\|^2 \rho_t \ dx  \right\}dt - \int_0^1 \int \frac{\delta}{\delta \rho_t} \mathcal{F}[\rho_t] \nabla \cdot( \nabla s_t \rho_t) \ dx \ dt \\
    &= \int_0^1 \left\{ \frac12 \int \| \nabla s_t\|^2 \rho_t \ dx  + \frac 12\int \| \nabla \frac{\delta}{\delta \rho_t} \mathcal{F}[\rho_t]\|^2 \rho_t \ dx  \right\}dt + \int_0^1 \int \frac{\delta}{\delta \rho_t} \mathcal{F}[\rho_t] \frac{\partial}{\partial t} \rho_t \ dx \ dt \\
    &= \int_0^1 \left\{ \frac12 \int \| \nabla s_t\|^2 \rho_t \ dx  + \frac 12\int \| \nabla \frac{\delta}{\delta \rho_t} \mathcal{F}[\rho_t]\|^2 \rho_t \ dx  \right\}dt + \int_0^1  \frac{d}{dt} \mathcal{F}[\rho_t]\ dt \\
    &= \int_0^1 \left\{ \frac12 \int \| \nabla s_t\|^2 \rho_t \ dx  + \frac 12\int \| \nabla \frac{\delta}{\delta \rho_t} \mathcal{F}[\rho_t]\|^2 \rho_t \ dx  \right\}dt + \mathcal{F}[\rho_1] - \mathcal{F}[\rho_0],
\end{align*}
where we expanded the square, integrated by parts, applied the continuity equation, and the chain rule. However, $\mathcal{F}[\rho_1]$ and $\mathcal{F}[\rho_0]$ are independent of the minimization, since the endpoints are fixed. Thus, minimizing $\mathcal{J}[\rho_t, s_t]$ is equivalent to minimizing 
\begin{align}
    \mathcal{S}[\rho_t,s_t] &= \int_0^1 \left\{ \mathcal{K}[\rho_t,s_t] - \mathcal{U}[\rho_t] \right\}dt\\
    &\coloneq\int_0^1 \left\{ \frac12 \int \| \nabla s_t\|^2 \rho_t \ dx  + \frac 12\int \| \nabla \frac{\delta}{\delta \rho_t} \mathcal{F}[\rho_t]\|^2 \rho_t \ dx \right\}dt,
\end{align}
which is precisely the Wasserstein Lagrangian action with canonical kinetic energy $\mathcal{K}[\rho_t, s_t]$\eqref{eq:canonical_kinetic_energy} and potential energy $\mathcal{U}[\rho_t] = -\frac 12 \int  \|\nabla \frac{\delta}{\delta \rho_t} \mathcal{F}[\rho_t] \|^2 \rho_t dx$. Thus, action minimization with respect to these Wasserstein energies is a necessary condition to be a gradient flow.

Now, given that the marginals $(\rho_t)_{0 \le t \le 1}$ minimize the action $\mathcal{S}$, we want to show that $(\rho_t)_{0 \le t \le 1}$ is a gradient flow of $\mathcal{F}$ if and only if we have the precise initial velocity $v_0 \propto -\nabla \frac{\delta \mathcal{F}[\rho_0]}{\delta \rho_0}$. Since a gradient flow must have $\nabla s_t(x) \propto -\nabla \frac{\delta}{\delta \rho_t}\mathcal{F}(\rho_t)(x)$ almost everywhere, the idea is to define a discrepancy function to clarify the analysis. Let $\lambda$ be the scalar proportion and 
\begin{equation}
    g_t(x) = v_t(x) + \lambda \nabla \frac{\delta \mathcal{F}[\rho_t]}{\delta \rho_t}(x).
\end{equation}
By definition, the dynamics are a gradient flow if and only if $g_t = 0$ for all $t \in [0, 1]$. We would then complete the proof of the proposition by showing that $\frac{d}{dt} g_t(x_t)$ defines a homogeneous linear evolution, since this would imply that the only valid solution is $g_0 = 0 \implies v_0 \propto -\nabla \frac{\delta \mathcal{F}[\rho_0]}{\delta \rho_0} $. The rest of the proof consists in computations to demonstrate this. We set $\lambda=1$ and note that the steps are equivalent for other values of $\lambda$. First, we compute
\begin{equation}
    \frac{d}{dt}g_t = \frac{d}{dt}v_t + \frac{d}{dt} \nabla \frac{\delta \mathcal{F}[\rho_t]}{\delta \rho_t}(x_t) = \frac{d}{dt}v_t + \nabla \frac{\partial}{\partial t} \frac{\delta \mathcal{F}[\rho_t]}{\delta \rho_t} + \nabla^2 \frac{\delta \mathcal{F}[\rho_t]}{\delta \rho_t} \cdot v_t.
    \label{eq:gt_evolution}
\end{equation}
We compute the first term by considering the Hamiltonian equation $\frac{d}{dt}v_t = -\nabla \frac{\delta \mathcal{U}}{\delta \rho_t}$ \eqref{eq:equations_of_motion} with the potential energy that we previously derived, $\mathcal{U}[\rho_t] = -\frac{1}{2} \int \|\nabla \frac{\delta \mathcal{F}[\rho_t]}{\delta \rho_t}\|^2 \rho_t \ dx$. We have that
\begin{align*}
\frac{\delta \mathcal{U}}{\delta \rho} = -\frac{1}{2}\|\nabla \frac{\delta \mathcal{F}}{\delta \rho}\|^2 + \int \frac{\delta^2 \mathcal{F}}{\delta \rho(x) \delta \rho(y)} \nabla \cdot (\rho_t(y) \nabla \frac{\delta \mathcal{F}}{\delta \rho_t}(y)) dy
\end{align*}
it follows that
\begin{equation}
    \frac{d}{dt}v_t = \nabla \left( \frac{1}{2} \left\| \nabla \frac{\delta \mathcal{F}[\rho_t]}{\delta \rho_t} \right\|^2 \right) - \nabla \int \frac{\delta^2 \mathcal{F}}{\delta \rho_t(x) \delta \rho_t(y)} \nabla \cdot \left( \rho_t(y) \nabla \frac{\delta \mathcal{F}[\rho_t]}{\delta \rho_t}(y) \right) dy.
\end{equation}
Then, using the continuity equation $\dot{\rho}_t = -\nabla \cdot (\rho_t v_t)$ and chain rule, the second term is
\begin{equation}
    \nabla \frac{\partial}{\partial t} \frac{\delta \mathcal{F}[\rho_t]}{\delta \rho_t} = \nabla \int \frac{\delta^2 \mathcal{F}}{\delta \rho_t(x) \delta \rho_t(y)} \dot{\rho}_t(y) dy = -\nabla \int \frac{\delta^2 \mathcal{F}}{\delta \rho_t(x) \delta \rho_t(y)} \nabla \cdot (\rho_t(y) v_t(y)) dy.
\end{equation}
We plug these terms back into the expression for $\frac{d}{dt}g_t$ \eqref{eq:gt_evolution}, separating all of the integral terms from the non-integral terms:
\begin{align}
    \frac{d}{dt}g_t &= A - B \\
    A & = \left[\nabla \left( \frac{1}{2} \left\| \nabla \frac{\delta \mathcal{F}[\rho_t]}{\delta \rho_t} \right\|^2 \right) + \nabla^2 \frac{\delta \mathcal{F}[\rho_t]}{\delta \rho_t} \cdot v_t \right]\\
    B & = \left[ \nabla \int \frac{\delta^2 \mathcal{F}}{\delta \rho_t(x) \delta \rho_t(y)} \nabla \cdot \left( \rho_t(y) \nabla \frac{\delta \mathcal{F}[\rho_t]}{\delta \rho_t}(y) \right) dy + \nabla \int \frac{\delta^2 \mathcal{F}}{\delta \rho_t(x) \delta \rho_t(y)} \nabla \cdot (\rho_t(y) v_t(y)) dy\right]
\end{align}
We simplify the first term in $A$ with the vector identity $\nabla (\frac12 \|\nabla f\|^2) = (\nabla^2f) \cdot \nabla f$ to factor out $g_t$ as desired:
\begin{align}
A &= \left(\nabla^2 \frac{\delta \mathcal{F}}{\delta \rho_t}\right) \cdot \nabla \frac{\delta \mathcal{F}}{\delta \rho_t} + \left(\nabla^2 \frac{\delta \mathcal{F}}{\delta \rho_t}\right) \cdot v_t \\ &= \left(\nabla^2 \frac{\delta \mathcal{F}}{\delta \rho_t}\right) \cdot \left( \nabla \frac{\delta \mathcal{F}}{\delta \rho_t} + v_t\right) = \left(\nabla^2 \frac{\delta \mathcal{F}}{\delta \rho_t}\right) \cdot g_t
\end{align}
We do a similar process for the integral terms in $B$, just by using the linearity of the divergence operator:
\begin{align}
    B &= \nabla \int \frac{\delta^2 \mathcal{F}}{\delta \rho_t(x) \delta \rho_t(y)} \left[ \nabla \cdot \left( \rho_t \nabla \frac{\delta \mathcal{F}}{\delta \rho_t} \right) + \nabla \cdot (\rho_t v_t) \right] dy \\ &= \nabla \int \frac{\delta^2 \mathcal{F}}{\delta \rho_t(x) \delta \rho_t(y)} \nabla \cdot \left( \rho_t(y) \left[ v_t(y) + \nabla \frac{\delta \mathcal{F}}{\delta \rho_t}(y) \right] \right) dy \\ &= \nabla \int \frac{\delta^2 \mathcal{F}}{\delta \rho_t(x) \delta \rho_t(y)} \nabla \cdot (\rho_t(y) g_t(y)) dy
\end{align}
Thus, $\frac{d}{dt}{g}_t=A-B$ is homogeneous and linear in $g_t$, which admits $0$ as solution. By uniqueness under standard regularity, we have that $g_t = 0$ for all $t \ge 0$ as desired.
\end{proof}

\TractableAccel* 

\begin{proof}
Since $p = \hat{p} = \frac{1}{N}\sum_{i=1}^{N}\delta_{x^{(i)}}$, it follows that the potential energy supported on the observed samples is equivalent to the true Wasserstein potential energy, i.e., $\Psi(x^{(1)}, \ldots, x^{(N)})= \mathcal{U}[p]$. 

We now prove the identity by considering a perturbation of sample $x^{(j)}$ in the direction of a vector $\eta \in \mathbb{R}^d$, and then matching the corresponding directional derivative for both $\Psi(x^{(1)}, \ldots, x_t^{(N)})$ and $\mathcal{U}[p]$. Indeed, equality implies that their directional derivatives are also the same.

First, we consider $\Psi$ and derive:
\begin{align*}
\frac{d}{d\epsilon}\bigg\rvert_{\epsilon=0}\Psi(x^{(1)}, \ldots, x^{(j)}+\epsilon \eta, \ldots, x^{(N)}) = \nabla_{x^{(j)}}\Psi(x^{(1)}, \ldots, x_t^{(N)}) \cdot \eta
\end{align*}

Then, to evaluate the analogous derivative for $\mathcal{U}[p]$, we first consider the perturbed measure:
\begin{align}
p_{\epsilon} = \frac{1}{N}\left(\sum_{i \neq j}^{N}\delta_{x^{(i)}} + \delta_{x^{(j)}+\epsilon \eta}\right)
\end{align}

The first order Taylor gives us:
\begin{align}
\mathcal{U}[p_\epsilon]-\mathcal{U}[p] = \int \frac{\delta}{\delta p}\mathcal{U}[p](x)\delta p(x)dx + o(\epsilon)
\end{align}

Plugging in our perturbation then yields:
\begin{align}
\mathcal{U}[p_\epsilon]-\mathcal{U}[p] &= \frac{1}{N}\int \frac{\delta}{\delta p}\mathcal{U}[p](x) (\delta_{x^{(j)}+\epsilon \eta}- \delta_{x^{(j)}})dx + o(\epsilon)\\
&=\frac{1}{N}\left( \frac{\delta}{\delta p}\mathcal{U}[p](x^{(j)}+\epsilon\eta) -\frac{\delta}{\delta p}\mathcal{U}[p](x^{(j)})\right) + o(\epsilon)
\end{align}

We then first order Taylor expand $\frac{\delta}{\delta p}\mathcal{U}[p](x^{(j)}+\epsilon\eta)]$ to get (using the chain rule):
\begin{align}
\mathcal{U}[p_\epsilon]-\mathcal{U}[p] &= \frac{1}{N}\left( \epsilon \nabla_x\frac{\delta}{\delta p}\mathcal{U}[p](x^{(j)})\cdot \eta + o(\epsilon)\right) + o(\epsilon)
\end{align}

Hence, it follows that:
\begin{align}
\frac{d}{d\epsilon}\bigg\rvert_{\epsilon=0}\mathcal{U}[p_\epsilon] = \frac{1}{N}\nabla_x\frac{\delta}{\delta p}\mathcal{U}[p](x^{(j)})\cdot \eta
\end{align}

Equating both sides gives us the desired (weak) equation:
\begin{align*}
\frac{d}{d\epsilon}\bigg\rvert_{\epsilon=0}\Psi(x^{(1)}, \ldots, x^{(j)}+\epsilon \eta, \ldots, x^{(N)}) &= \frac{d}{d\epsilon}\bigg\rvert_{\epsilon=0}\mathcal{U}[p_\epsilon]\\
\implies \nabla_{x^{(j)}}\Psi(x^{(1)}, \ldots, x_t^{(N)}) \cdot \eta&= \frac{1}{N}\nabla_x\frac{\delta}{\delta p}\mathcal{U}[p](x^{(j)})\cdot \eta
\end{align*}
\end{proof}

\section{Non-identifiability of Individual Trajectories}
\label{app:non_iden_traj}

While \cref{prop:unique_conservative_vector_field} shows that the canonical continuity equation \eqref{eq:canonical_kinetic_energy} can without loss of generality model any continuous evolution of marginals, we note that these marginals can generally be produced by many laws on paths. As a simple example, consider trajectories from a gradient-flow diffusion,
\begin{align}
    \dXt = -\nabla \Psi(X_t)\dt + \sigma \dWt.
    \label{eq:gradient_flow_sde}
\end{align}
Then, its population dynamics $(p_t)_{0 \le t \le 1}$ are determined by its Fokker-Planck equation, which can equivalently be expressed via the continuity equation \eqref{eq: continuity_eq} for some corresponding transport vector field $\nabla s_t$. In fact, it is well known that $\nabla s_t$ is determined by 
\begin{align*}
    \nabla s_t(x) = -\nabla \Psi(x) - \frac{\sigma^2}{2}\nabla \log \rho_t(x) = -\nabla \left( \Psi(x) + \frac{\sigma^2}{2}\log \rho_t(x)\right).
\end{align*}
Thus, when initialized at $p_0$, the ODE $\frac{d}{dt}x_t = \nabla s_t(x)$ will produce trajectories, which, when aggregated, produce the marginals $(p_t)_{0 \le t \le 1}$. However, these ODE trajectories will differ compared to the stochastic trajectories from the diffusion process \eqref{eq:gradient_flow_sde}. Similarly, non-conservative vector fields can produce the same population dynamics $(p_t)_{0 \le t \le 1}$. 

Additional knowledge is therefore needed to uniquely identify the law on paths from observed population dynamics $(p_t)_{0 \le t \le 1}$. Indeed, this is the central problem for SDE identifiability \citep{weinreb2018fundamental, lavenant2021towards, guan2024identifying, guangradient}. Recent work has shown that if we restrict candidate laws to gradient-flow SDEs with drift-diffusion $(-\nabla \Psi, \sigma^2)$, then as long as the observed marginals $(p_t)_{0 \le t \le 1}$ are not trivially at equilibrium, then the SDE parameters that identify the law on paths is in fact identifiable \citep{guangradient}. Similar identifiability results hold when the class of SDEs is assumed to be linear and when the marginals obey certain symmetry-breaking conditions  \citep{guan2024identifying}. However, in this work, we do not assume anything about the underlying paths. By only considering the evolution of marginals, we always evolve trajectories according to the canonical vector field in \eqref{eq: continuity_eq} for our theoretical results.

\section{Additional Experiment Details}
\label{sec:additional_experiment_appendix}
We present additional details about the datasets, implementation of \texttt{WLM} and baseline methods, and analysis of results for each of our four experiments. 

\paragraph{WLM Default Hyperparameters.}
In all experiments, we implement the potential energy model $\Psi_\theta: \R^{N \times d} \to \R$ as a neural network equipped with both attention and feed forward layers. We found that the attention mechanism helped the model learn more expressive interacting dynamics. Our default settings used $4$ attention heads and $4$ stacked blocks. Each block consists of a multi-head self-attention mechanism followed by a feed forward network. The default hidden dimensions is $64$ and the default inner dimension from the feed forward network is $512$. We used a learning rate of $lr=1e-4$ for $\Psi_\theta$ in all experiments. We detail changes from the default hyperparameters in each experiment subsection. We also note that since \texttt{WLM} is simulation-based, one important hyperparameter is the number of substeps used between marginals during training. While we always sample $5$ substeps at inference time to faithfully simulate from the finalized learned model, we interestingly found that using only $1$ substep between marginal often significantly improved the model, so we use this as default. To implement our loss, we by default use the Sinkhorn loss with $p=2$ from the geomloss package \url{https://www.kernel-operations.io/geomloss/}. The kernel blur parameter is estimated from the data. The \texttt{WLM} code is implemented in PyTorch. 

\paragraph{Hardware.}
All experiments were conducted on an HPC cluster, primarily on NVIDIA GPUs (such as the A100 or RTX 8000 series) on one GPU and $2-4$ CPU cores, with between $10-50$GB of memory per run. Depending on the size of the dataset, most experiments completed in between $3-15$ hours. We present the runtime breakdown per experiment below, noting that time complexity principally depends on the number of particles and training marginals, since \texttt{WLM} is trained via simulated rollouts. We explore the effect of mini-batching on runtime and performance for the Embryoid Body experiment in \cref{fig:w1_vs_runtime}.

\begin{table}[htbp]
    \centering
    \caption{Runtime per experiment}
    \vspace{0.1in}
    \begin{tabular}{l c c r r}
        \toprule
        \textbf{Experiment} & \textbf{Particles at $t=0$} & \textbf{Number of training marginals} & \textbf{Epochs} & \textbf{Hours} \\
        \midrule
        GF SDEs        & 1000 & 9  & 100,000 & 7.00 ($\pm$0.82) \\
        Boids          & 1000 & 49 & 30,000  & 10.88 ($\pm$1.38) \\
        Oceans (small) & 111  & 4  & 50,000  & 1.32 ($\pm$0.04) \\
        Oceans (big)   & 400  & 9  & 50,000  & 1.99 ($\pm$0.21) \\
        EB             & 1024 (minibatched) & 4  & 10,000  & 0.35 ($\pm$0.09) \\
        \bottomrule
    \end{tabular}
    \label{tab:experiment_metrics}
\end{table}

\subsection{Gradient-flow SDEs}
\label{sec:gradient_flows_appendix}
We consider the same five two-dimensional potentials used in the main experiments by \citet{guangradient}:
\begin{align}
\text{\textbf{Bohachevsky}} \qquad
V(x) &= 10\!\left(x_1^2 + 2x_2^2 - 0.3\cos(3\pi x_1) - 0.4\cos(4\pi x_2)\right)
\label{eq:pot-bohachevsky} \\[1em]
\text{\textbf{Oakley--O'Hagan}} \qquad
V(x) &= 5 \sum_{i=1}^2 \left(\sin(x_i) + \cos(x_i) + x_i^2 + x_i\right)
\label{eq:pot-oakley-ohagan} \\[1em]
\text{\textbf{Quadratic}} \qquad
V(x) &= 5\,\|x\|^2
\label{eq:pot-quadratic} \\[1em]
\text{\textbf{Styblinski--Tang}} \qquad
V(x) &= \tfrac{1}{2}\sum_{i=1}^2 \left(x_i^4 - 16x_i^2 + 5x_i\right)
\label{eq:pot-styblinski-tang} \\[1em]
\text{\textbf{Wavy plateau}} \qquad
V(x) &= \sum_{i=1}^2 \left( \cos(\pi x_i) + \tfrac{1}{2}x_i^4 - 3x_i^2 + 1 \right)
\label{eq:pot-wavy-plateau}
\end{align}
Data is then generated by choosing diffusivity $\sigma^2=1$. We sample $10$ marginals with equal spacing $\Delta t = 0.01$, initialized from a mean $0$ Gaussian, $p_0 \sim \mathcal{N}(0, 0.2)$. We generated the data using the Euler-Maruyama scheme, with $10$ substeps per $\Delta t$, i.e. $\Delta t_{EM} = 0.001$. Each of these SDEs converge to their stationary distribution at different rates and thus offers a range of different geometries and dynamics. We note that the Bohachevsky SDE almost immediately converges to its stationary distribution, whereas the Oakley-O'Hagan SDE for example is far from reaching its stationary distribution after $10$ steps.

\paragraph{WLM Implementation.}
We use the default hyperparameters and set the number of epochs to $100,000$. The friction learning rate is set to $1e-2$ when friction is learnable, and set to initial value $\gamma=1$.

\paragraph{Implementation of Baselines.}
We use default implementations for both \textsc{JKOnet}$^\ast$ and \textsc{nn-APPEX}, as done for experiments on the same benchmark dataset. See \citep[Appendix C]{terpin2024learning} and \citep[Section 6]{guangradient} respectively.

\begin{table*}[t]
\centering
\caption{Comparison of average $W_1$ distances for \textsc{JKOnet}*, \textsc{nn-APPEX}, and \texttt{WLM} across 5 Gradient-Flow SDEs in Paired and Unpaired settings for train (first $10$ marginals) and test (next $10$ marginals). }
\label{tab:full_gf_sde_results}
\resizebox{\textwidth}{!}{%
\begin{tabular}{llcccc}
\toprule
 &  & \multicolumn{2}{c}{\textbf{Paired}} & \multicolumn{2}{c}{\textbf{Unpaired}} \\
\cmidrule(lr){3-4} \cmidrule(lr){5-6}
\textbf{SDE} & \textbf{Method} & Train $W_1$ ($\pm$SE) & Test $W_1$ ($\pm$SE) & Train $W_1$ ($\pm$SE) & Test $W_1$ ($\pm$SE) \\
\midrule

\multirow{4}{*}{Bohachevsky} 
 & \textsc{JKOnet}* & $0.0814 \pm 0.0033$ & $0.0969 \pm 0.0018$ & $0.1969 \pm 0.0128$ & $0.3393 \pm 0.0149$ \\
 & \textsc{nn-APPEX} & $0.0798 \pm 0.0021$ & $0.0917 \pm 0.0008$ & $0.0696 \pm 0.0029$ & \textbf{0.0742} $\pm$ \textbf{0.0017} \\
 & \texttt{WLM} (learnable friction) & \textbf{0.0388} $\pm$ \textbf{0.0007} & \textbf{0.0488} $\pm$ \textbf{0.0011} & \textbf{0.0660} $\pm$ \textbf{0.0022} & $0.0907 \pm 0.0023$ \\
 & \texttt{WLM} ($0$ friction) & $0.3195 \pm 0.0389$ & $0.3838 \pm 0.0460$ & $0.2690 \pm 0.0323$ & $0.3241 \pm 0.0262$ \\
\midrule

\multirow{4}{*}{Oakley O'Hagan} 
 & \textsc{JKOnet}* & $0.0578 \pm 0.0025$ & $0.2067 \pm 0.0292$ & $0.1951 \pm 0.0340$ & $1.7709 \pm 0.3550$ \\
 & \textsc{nn-APPEX} & $0.0526 \pm 0.0011$ & $0.0514 \pm 0.0018$ & $0.0729 \pm 0.0024$ & $0.1375 \pm 0.0179$ \\
 & \texttt{WLM} (learnable friction) & \textbf{0.0415} $\pm$ \textbf{0.0006} & \textbf{0.0509} $\pm$ \textbf{0.0013} & \textbf{0.0430} $\pm$ \textbf{0.0005} & \textbf{0.0734} $\pm$ \textbf{0.0026} \\
 & \texttt{WLM} ($0$ friction) & $0.0704 \pm 0.0062$ & $0.2462 \pm 0.0370$ & $0.0616 \pm 0.0032$ & $0.2470 \pm 0.0396$ \\
\midrule

\multirow{4}{*}{Quadratic} 
 & \textsc{JKOnet}* & $0.0723 \pm 0.0046$ & $0.0753 \pm 0.0015$ & $0.2665 \pm 0.0415$ & $0.7161 \pm 0.0453$ \\
 & \textsc{nn-APPEX} & $0.0432 \pm 0.0015$ & \textbf{0.0365} $\pm$ \textbf{0.0005} & $0.0538 \pm 0.0019$ & $0.0631 \pm 0.0028$ \\
 & \texttt{WLM} (learnable friction) & \textbf{0.0339} $\pm$ \textbf{0.0008} & $0.0488 \pm 0.0021$ & \textbf{0.0391} $\pm$ \textbf{0.0009} & \textbf{0.0555} $\pm$ \textbf{0.0024} \\
 & \texttt{WLM} ($0$ friction) & $0.0505 \pm 0.0021$ & $0.0964 \pm 0.0085$ & $0.0651 \pm 0.0035$ & $0.1089 \pm 0.0062$ \\
\midrule

\multirow{4}{*}{Styblinski-Tang} 
 & \textsc{JKOnet}* & $0.1359 \pm 0.0186$ & $0.3791 \pm 0.0386$ & $0.2381 \pm 0.0409$ & $2.0202 \pm 0.3516$ \\
 & \textsc{nn-APPEX} & 0.1142 $\pm$ 0.0124 & \textbf{0.2643} $\pm$ \textbf{0.0206} & $0.2066 \pm 0.0243$ & $0.7183 \pm 0.0746$ \\
 & \texttt{WLM} (learnable friction) & $0.1164 \pm 0.0134$ & $0.3983 \pm 0.0505$ & {0.1147} $\pm$ {0.0132} & \textbf{0.7091} $\pm$ \textbf{0.0917} \\
 & \texttt{WLM} ($0$ friction) & \textbf{0.0861} $\pm$ \textbf{0.0075} & $0.3898 \pm 0.0724$ & \textbf{0.1134} $\pm$ \textbf{0.0119} & $0.7271 \pm 0.1101$ \\
\midrule

\multirow{4}{*}{Wavy Plateau} 
 & \textsc{JKOnet}* & 0.0759 $\pm$ 0.0046 & $0.2048 \pm 0.0293$ & $0.2827 \pm 0.0699$ & $3.2447 \pm 0.5395$ \\
 & \textsc{nn-APPEX} & $0.1116 \pm 0.0127$ & $0.2128 \pm 0.0071$ & $0.1085 \pm 0.0064$ & $0.3090 \pm 0.0298$ \\
 & \texttt{WLM} (learnable friction) & $0.0770 \pm 0.0064$ & \textbf{0.1383} $\pm$ \textbf{0.0057} & \textbf{0.0750} $\pm$ \textbf{0.0054} & \textbf{0.3002} $\pm$ \textbf{0.0306} \\
 & \texttt{WLM} ($0$ friction) & \textbf{0.0671} $\pm$ \textbf{0.0039} & $0.1598 \pm 0.0186$ & $0.0879 \pm 0.0058$ & $0.3249 \pm 0.0415$ \\
\bottomrule
\end{tabular}%
}
\end{table*}

\begin{figure}
    \centering
    \begin{subfigure}{0.45\textwidth}
        \centering
        \includegraphics[width=\linewidth]{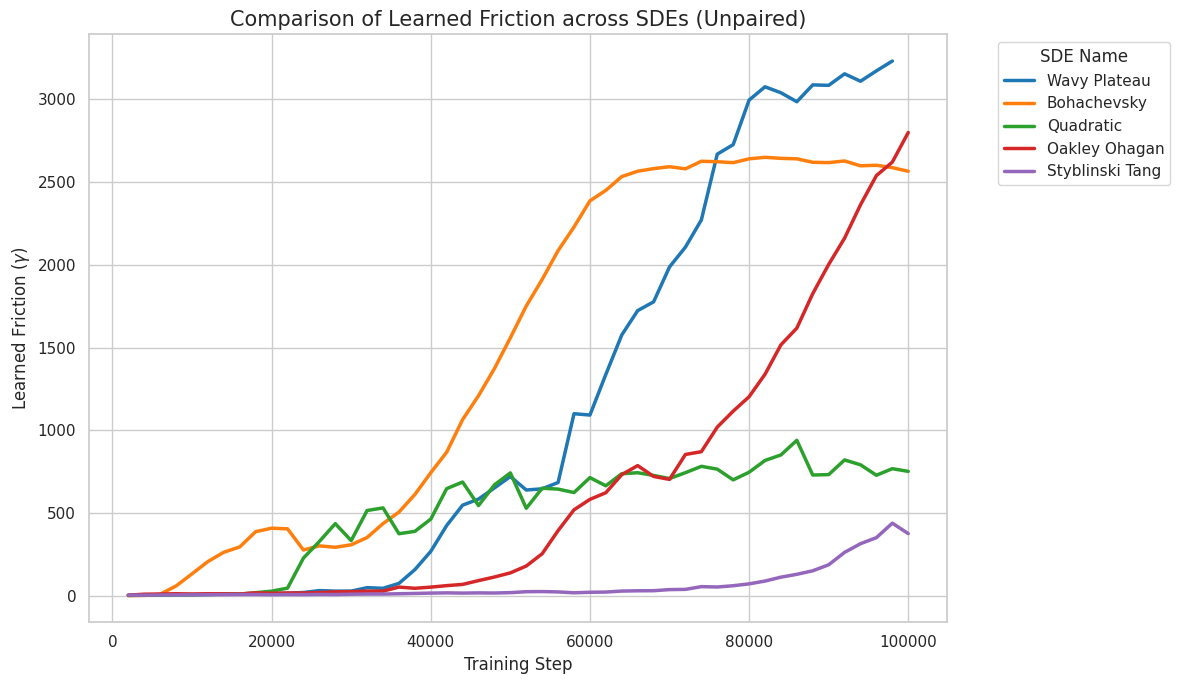}
        \caption{Learnable friction for gradient flows (unpaired setting)}
        \label{fig:friction_kill}
    \end{subfigure}
    \hfill
    \begin{subfigure}{0.45\textwidth}
        \centering
        \includegraphics[width=\linewidth]{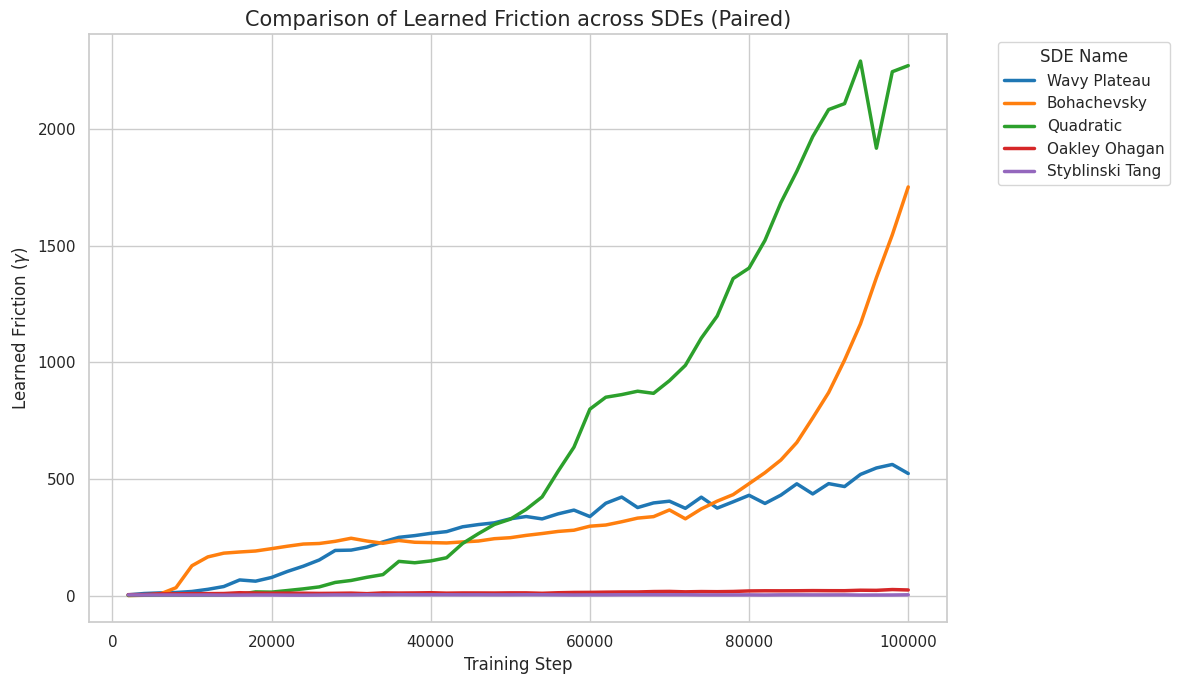}
        \caption{Learnable friction for  gradient flows (paired setting). }
        \label{fig:friction_nokill}
    \end{subfigure}
    
    \caption{\texttt{WLM} Learned Friction ($\gamma$) curve for each SDEs over $100,000$ training epochs. Minimum learned friction at $100k$ epochs is $\gamma =5.28$ (Styblinski-Tang paired setting), which collapses to the equilibrium distribution almost instantaneously.}
    \label{fig:friction_comparison_gf_sde}
\end{figure}

\subsection{Ocean Currents}
\label{sec:ocean_currents_appendix}

\paragraph{Dataset Description.}
As done in \citet{shen2024multi, berlinghieri2025oh, petrovic2025curly}, we obtain real ocean vortex data from the Gulf of Mexico, from HYbrid Coordinate Ocean
Model (HYCOM) \url{https://www.hycom.org/data/gomb0pt01/gom-reanalysis}. We follow the same data preprocessing as \citet{shen2024multi, petrovic2025curly} to reproduce the small vortex, and as done in \citet{berlinghieri2025oh} to reproduce the big vortex. 

\paragraph{WLM Implementation.}
To implement WLM for interpolation on both vortices, we use one attention head, and for the architecture, we increased the number of hidden dimensions to $256$ and reduced the feed forward inner dimension to $64$. We set learnable friction in both settings (with learning rate $1e-3$), but only model time as a feature for interpolation (small vortex), using $16$ time features to do so. For both tasks, we train for $50,000$ epochs and use EMA $0.9999$. We use weight decay $1e-2$ for the interpolation task on the small vortex and weight decay $1e-3$ for the forecast task on the big vortex.


\paragraph{Implementation of Baselines.}
With one parameter exception, we use the suggested implementations for both \textsc{SBIRR} and \textsc{Snap-MMD} for the oceans dataset used for interpolation and forecasting experiments in \citet{shen2024multi}
and \citet{berlinghieri2025oh} respectively. Due to memory overflow, to implement the iterative proportion fitting SB-solver in \textsc{SBIRR}, we use $10$ time substeps rather than $50$ substeps. We note that the corresponding results are comparable and sometimes better than reported in the original papers \citep{shen2024multi, berlinghieri2025oh}. 

For the action matching \citep{neklyudov2023action} and DICE \citep{blickhan2025dice} baselines on the oceans dataset, we use a $2$-layer MLP with hidden dimension $64$ and SiLU activations. We then ran each method using Adam with a constant learning rate of $10^{-4}$ for $2{,}000$ iterations (\textsc{AM}, \textsc{UAM}, \textsc{sAM}) and $5000$ iterations (\textsc{DICE}), with batch size $256$ and gradient clipping at norm $1$. For \textsc{sAM} we use diffusion coefficient $\sigma=0.1$. We note that DICE \citep{blickhan2025dice} introduced a more stable action matching loss over discrete observations, which mitigates the over-optimization of a space-independent, but time-inhomogeneous potential. Despite this, we observe that \textsc{DICE} also diverges significantly during training, likely due to the sparsity of the data, which has only $111$ particles per marginal, with large time gaps between observations. We therefore report \textsc{DICE} at its best checkpoint (epoch $3500$) rather than at the end of training.

\subsection{Embryonic sc-RNA data}
\label{sec:cells_appendix}

For the EB dataset \citep{moon2019visualizing}, we follow standard pre-processing of $5$ dimensional PCA and whitening the data, as first done in \citet{tong2020trajectorynet}, and subsequently replicated in many experiments \citep{neklyudov2023computational, kapusniak2024metric, tong2023simulation}. For leave-one-out interpolation, we train $3$ separate models, using the remaining $4$ marginals to predict the last one. We use the provided RNA velocity to simulate our rollout from $p_0$ during training, and to simulate from the previous observed marginal $p_{h-1}$ when interpolating the holdout, but we do not notice a significant difference between this choice and starting with $v_0=0$. Indeed, it has been conjectured that the velocity for this dataset may not be highly informative \citep{tong2020trajectorynet}, and we also recover moderate to high dissipation, which places less emphasis on the importance of the initial velocities.

\paragraph{WLM implementation.}
We modify the default architecture by dropping the feed forward inner dimension to $256$. We also use weight decay $1e-3$ and dropout $0.1$. For the training loss, we use the Sinkhorn divergence loss with $p=1$. For learnable friction, we initialize $\gamma=1$ and friction learning rate $1e-2$. We run for $10,000$ epochs and use an EMA of $0.999$.

\paragraph{Implementation of baselines.}

We implement the action matching \citep{neklyudov2023action} and DICE \citep{blickhan2025dice} baselines with the same implementation specified for the ocean vortex interpolation. We report \textsc{DICE} at its best checkpoint (epoch $3000$) rather than at the end of training.

\begin{figure}[h!] 
    \centering
    \includegraphics[width=0.5\linewidth]{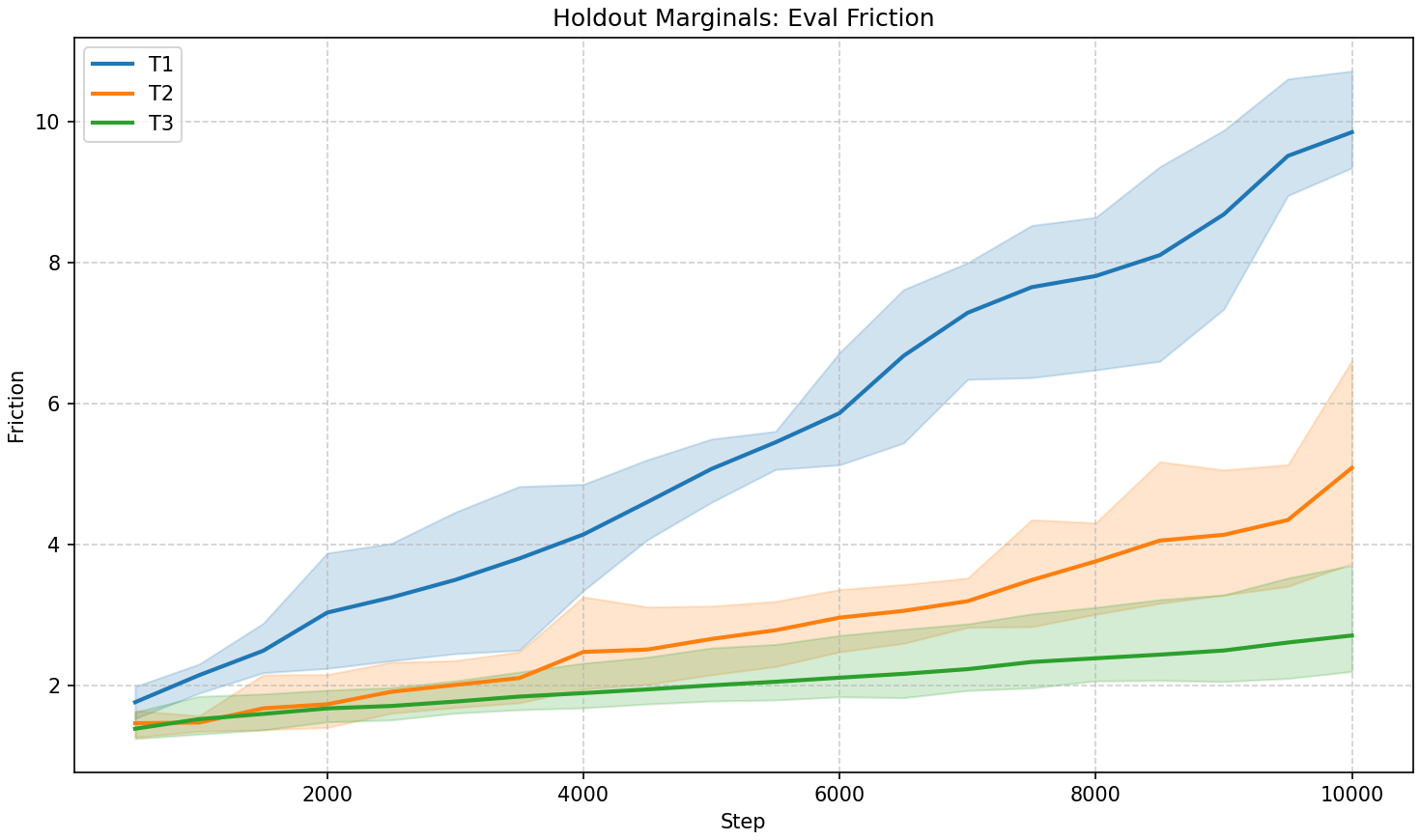} 
    \caption{Learnable friction on the Embryoid Body (EB) dataset for different holdout times.}
    \label{fig:eb_friction_holdout}
\end{figure}

\begin{figure}[h!] 
    \centering
    \includegraphics[width=0.5\linewidth]{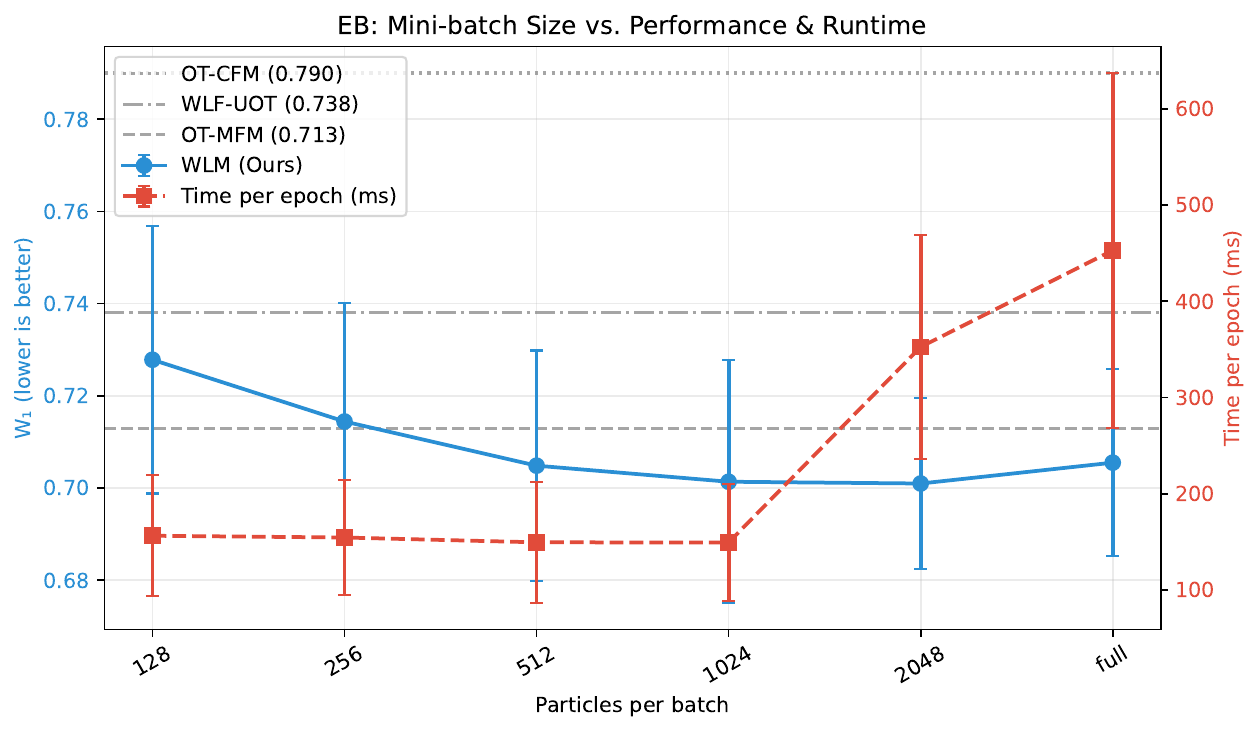} 
    \caption{Runtime vs. performance when performing mini-batching with \texttt{WLM} averaged over $3$ random seeds.}
    \label{fig:w1_vs_runtime}
\end{figure}

\subsection{Boids}
\label{sec:boids_appendix}
We implement the Boids algorithm using a Vicsek-style interacting particle system based on the classic Boids algorithm. By default, we simulate $N=1000$ agents in $\R^2$ adhering to local interaction rules, which is essentially the three Boids interaction rules with a boundary condition. In particular, at each time step $t$, the velocity is updated based on the state of its neighbours, and the acceleration is determined by the weighted sum of four distinct forces:
\newpage
\begin{itemize}
    \item Separation: short-range repulsion to avoid colliding wth flockmates. Agents move away from flockmates within an inner radius $R_{inner}$.
    \item Alignment: velocity matching with flockmates. Steer towards average velocity difference of all neighbours within an outer radius $R_{outer}$.
    \item Cohesion: Long-range attraction. Steer towards average relative position of flockmates within $R_{outer}$.
    \item Boundary: Restorative force applied when a Boid exceeds the simulation boundary, to steer back to origin.
\end{itemize}
Our default parameters are $R_{inner}=0.3$, $R_{outer}=1.0$, and different weights for enacted forces based on the four rules. By default, separation force is weighted $0.1$, alignment is weighted $0.3$, cohesion is weighted $0.005$, and boundary is weighted $0.5$, and triggered at magnitude $5.0$ away from the origin.

\begin{table}[!t]
\centering
\caption{Boids: Average $W_1$ distances between true and learned marginals for different initializations ($v_0$) on the same random seed. We distinguish between train times (first $50$) and forecast (next $50$).}
\label{tab:boids_v0_experiment}

\footnotesize          
\renewcommand{\arraystretch}{0.85} 
\setlength{\tabcolsep}{4pt} 

\begin{tabular}{lcc}
\toprule
\textbf{$v_0$} & Train $W_1$ ($\pm$SE) & Forecast $W_1$ ($\pm$SE) \\
\midrule
True & $0.530 \pm 0.034$ & $1.523 \pm 0.051$ \\
DICE & $0.684 \pm 0.027$ & $1.521 \pm 0.044$ \\
Zero & $0.726 \pm 0.026$ & $1.645 \pm 0.052$ \\
\bottomrule
\end{tabular}
\end{table}

\paragraph{WLM Implementation.}
We use the default hyperparameters with $p=1$ for the Sinkhorn divergence, and run for $30,000$ epochs. 

\paragraph{Implementation of Baselines.}
We use default implementations for both \textsc{JKOnet}$^\ast$ and \textsc{nn-APPEX}, as done for experiments on the same benchmark dataset. See \citep[Appendix C]{terpin2024learning} and \citep[Section 6]{guangradient} respectively.

\begin{figure*}[t] 
    \centering
    
    \def\rowTop{0.73}
    \def\rowBot{0.23}
    \def\labelSep{-0.015} 
    \def\colY{1.05} 

    \begin{subfigure}{0.965\linewidth}
        \centering
        \begin{tikzpicture}
            \node[anchor=south west, inner sep=0] (image) at (0,0) {
                \includegraphics[width=\linewidth]{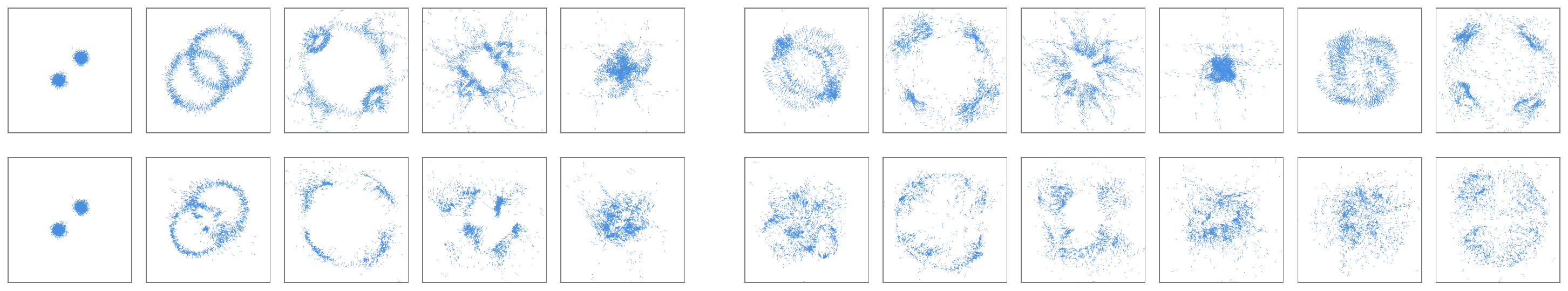}
            };
            
            \begin{scope}[x={(image.south east)}, y={(image.north west)}]
                
                \node[rotate=90, anchor=south, font=\tiny] at (\labelSep, \rowTop) {Ground Truth};
                \node[rotate=90, anchor=south, font=\tiny] at (\labelSep, \rowBot) {\scalebox{1.3}{\texttt{WLM}}};

                \foreach \x/\t in {0.044/0.0, 0.133/5.0, 0.221/10.0, 0.310/15.0, 0.398/20.0} {
                    \node[anchor=south, font=\tiny] at (\x, \colY) {$t=\t$};
                }
                \foreach \x/\t in {0.513/25.0, 0.602/30.0, 0.690/35.0, 0.779/40.0, 0.867/45.0, 0.956/50.0} {
                    \node[anchor=south, font=\tiny] at (\x, \colY) {$t=\t$};
                }
            \end{scope}
        \end{tikzpicture}
        \caption{Generalization to unseen two component Gaussian mixture}
        \label{fig:app_boids_1}
    \end{subfigure}
    
    \vspace{0.5em}

    \begin{subfigure}{0.965\linewidth}
        \centering
        \begin{tikzpicture}
            \node[anchor=south west, inner sep=0] (image) at (0,0) {
                \includegraphics[width=\linewidth]{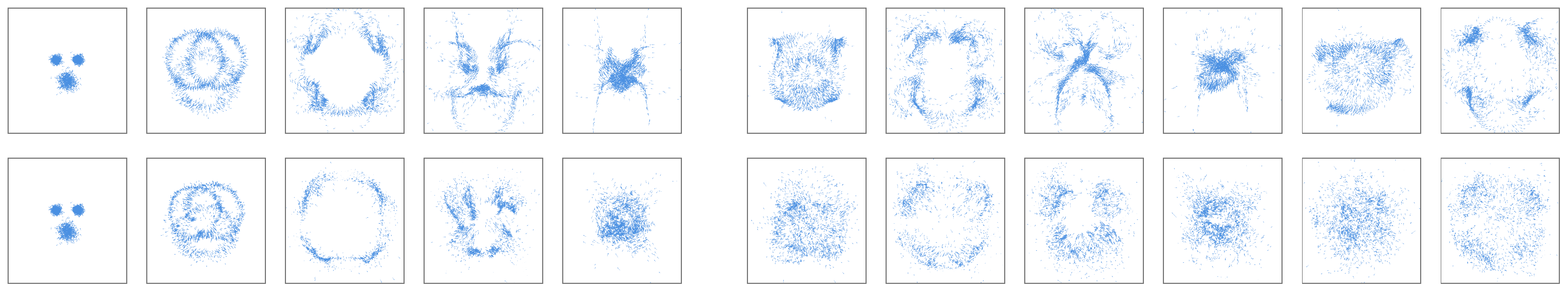}
            };
            
            \begin{scope}[x={(image.south east)}, y={(image.north west)}]
                \node[rotate=90, anchor=south, font=\tiny] at (\labelSep, \rowTop) {Ground Truth};
                \node[rotate=90, anchor=south, font=\tiny] at (\labelSep, \rowBot) {\scalebox{1.3}{\texttt{WLM}}};
            \end{scope}
        \end{tikzpicture}
        \caption{Generalization to unseen three component Gaussian mixture}
        \label{fig:app_boids_2}
    \end{subfigure}

    \vspace{0.5em}

    \begin{subfigure}{0.965\linewidth}
        \centering
        \begin{tikzpicture}
            \node[anchor=south west, inner sep=0] (image) at (0,0) {
                \includegraphics[width=\linewidth]{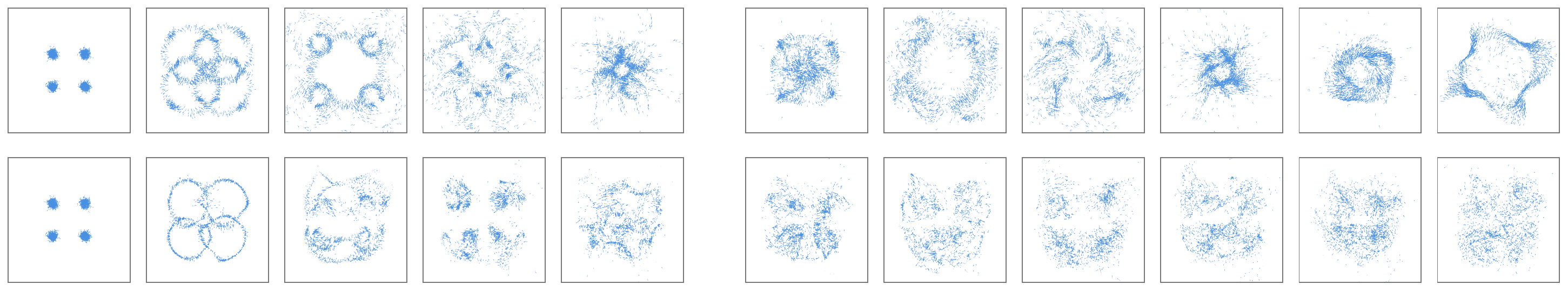}
            };
            
            \begin{scope}[x={(image.south east)}, y={(image.north west)}]
                \node[rotate=90, anchor=south, font=\tiny] at (\labelSep, \rowTop) {Ground Truth};
                \node[rotate=90, anchor=south, font=\tiny] at (\labelSep, \rowBot) {\scalebox{1.3}{\texttt{WLM}}};
            \end{scope}
        \end{tikzpicture}
        \caption{Generalization to unseen four component Gaussian mixture}
        \label{fig:app_boids_3}
    \end{subfigure}

    \caption{
        \textbf{Predicting Boids on unseen dynamics:} 
        Qualitative comparison of ground truth Boids dynamics (top row of each panel) versus the predicted \texttt{WLM} dynamics (bottom row of each panel) for three unseen Gaussian mixture initial distributions. \texttt{WLM} was trained on $50$ frames whose population was a centered Gaussian (see \cref{fig:boids_comparison}).
    }
    \label{fig:app_boids_generalization}
\end{figure*}
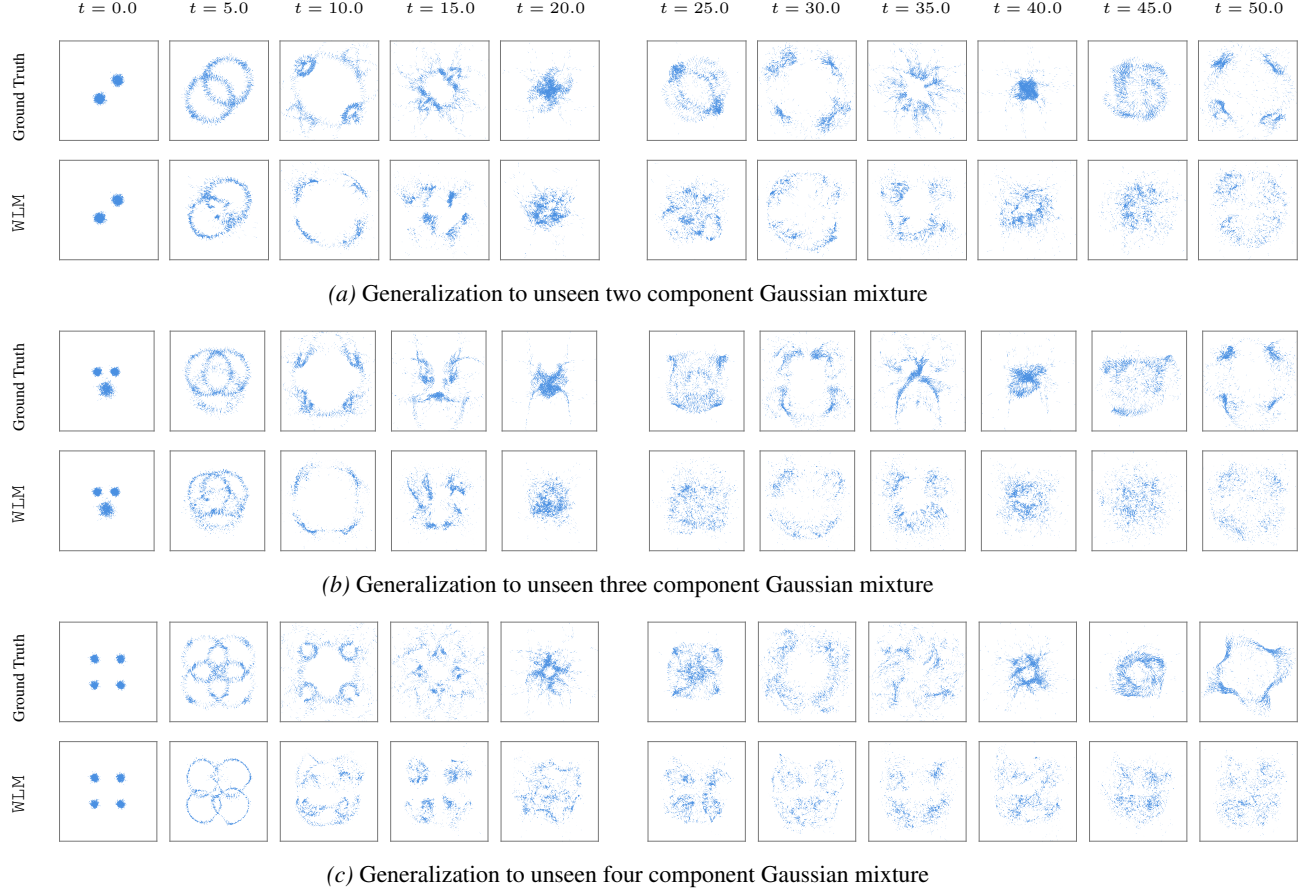


\end{document}